\documentclass[conference]{IEEEtran}

\usepackage{cite}
\usepackage{amsmath}
\usepackage{url}
\usepackage{mathptmx}
\usepackage{amsthm}
\usepackage{amssymb}
\usepackage{graphicx}
\usepackage{booktabs}
\usepackage{multirow}
\usepackage{enumitem}
\usepackage{hyperref}
\usepackage{subcaption}
\usepackage{algorithm}
\usepackage{algpseudocode}
\usepackage{framed}
\usepackage{siunitx}
\newtheorem{theorem}{Theorem}

\usepackage{tcolorbox}
\usepackage{makecell}
\usepackage{bbm}

\newcommand{\rqtakeaway}[2]{%
\vspace{0.4em}
\noindent\fbox{%
\begin{minipage}{\dimexpr\linewidth-2\fboxsep-2\fboxrule\relax}
\textbf{#1.} #2
\end{minipage}}
\vspace{0.4em}
}

\begin{document}

\title{Thermo-FL: Thermal-Aware Robust Federated Fine-Tuning of Large Language Models for Edge AI}

\author{
\IEEEauthorblockN{
Shiva Shrestha\IEEEauthorrefmark{1},
Kazi Shaharair Sharif\IEEEauthorrefmark{1},
Zongxing Xie\IEEEauthorrefmark{1},
Jiajing Huang\IEEEauthorrefmark{2},
Anhao Xiang\IEEEauthorrefmark{3},
and Honghui Xu\IEEEauthorrefmark{3}
}

\IEEEauthorblockA{
\IEEEauthorrefmark{1}
Department of Computer Science, Kennesaw State University\\
Marietta, GA 30060, USA\\
\texttt{\{sshres16,ksharif2\}@students.kennesaw.edu}\\
\texttt{zxie1@kennesaw.edu}
}

\IEEEauthorblockA{
\IEEEauthorrefmark{2}
School of Data Science and Analytics, Kennesaw State University\\
Marietta, GA 30060, USA\\
\texttt{jhuang24@kennesaw.edu}
}

\IEEEauthorblockA{
\IEEEauthorrefmark{3}
Department of Information Technology, Kennesaw State University\\
Marietta, GA 30060, USA\\
\texttt{\{axiang,hxu10\}@kennesaw.edu}
}
}

\maketitle
\begin{abstract}
Federated fine-tuning enables large language models to adapt on edge devices without centralizing private data, but practical deployments must address hardware instability and adversarial update corruption together. Thermally constrained clients may throttle, slow local training, or delay synchronous aggregation, while Byzantine clients and communication-layer adversaries can corrupt the updates used to form the global model. To address these challenges, we present Thermo-FL, a thermal-aware federated LoRA fine-tuning framework that uses device temperature as an active control signal for local adapter training and sparse update transmission. On the client side, Thermo-FL adjusts the active LoRA-layer fraction and transmitted update density as devices heat or cool, reducing workload under thermal stress. On the server side, Thermo-FL introduces TERRA, a robust aggregation pipeline for dynamically sparse LoRA updates that combines norm filtering, mask-aware directional validation, adaptive active-coordinate clipping, and mask-aware aggregation. We evaluate Thermo-FL using both a large-scale emulator and a Jetson-based physical testbed. In the emulator, Thermo-FL improves robustness under adversarial sparse aggregation and achieves the strongest BoolQ accuracy across clean and attack settings while remaining competitive on GSM8K. In the physical prototype, Thermo-FL stabilizes device temperature, reduces compressed upload size through bitmap sparse encoding, and preserves GSM8K utility under sign-flip/scale and MITM perturbations. These results show that secure edge LLM adaptation should jointly consider hardware behavior, workload regulation, sparse communication, and aggregation robustness. 
\end{abstract}

\IEEEpeerreviewmaketitle

\section{Introduction}

The increasing deployment of large language models (LLMs) has brought sophisticated language understanding and generation capabilities directly to edge devices, such as mobile phones, drones, robots, vehicles, medical systems, and industry sensors~\cite{xu2024ondevicelanguagemodelscomprehensive, xu2025surveyprivacysecuritymobile}. This shift creates a rapidly growing need to personalize and adapt the LLMs to domain-specific tasks using the data generated outside the centralized cloud server environments. For individual users, local data may capture preferences, writing style, language, and interaction patterns, and for small organizations, clinics, field teams, and industrial operators, local data may contain domain vocabulary, operational procedures, sensitive records, internal reports, and so on that are difficult to centralize because of privacy, cost, and government constraints~\cite{EdgeLLMsurvey}. This need is already visible in large-scale mobile systems, where federated learning (FL) has been used to improve the on-device keyboard and language model behaviour without collecting users' raw data text~\cite{hard2018federated}. Federated Learning offers a promising path for this setting by allowing multiple clients to adapt the shared model collaboratively without having to upload the raw data to the central server~\cite{mcmahan2023communicationefficientlearningdeepnetworks}. In particular, combined with the parameter-efficient fine-tuning (PEFT) techniques like Low-Rank Adaptation (LoRA), federated learning allows the resource-constrained devices to collaboratively adapt the LLMs without updating the entire model~\cite{hu2021loralowrankadaptationlarge}.

However, federated learning on the edge devices is not a simple scaled-down version of the cloud-based distributed training~\cite{xia2021survey}. Unlike cloud servers, edge devices operate under limited power, constrained cooling, and unstable runtime conditions, which can delay the training process, reduce client participation, and increase the communication overhead. At the same time, unlike cloud servers running in the controlled data center, edge devices are deployed in the physically open and often unmonitored environments, which may expose the federated process to adversarial threats, including but not limited to corrupted client updates and communication-layer perturbations~\cite{bagdasaryanBackdoorFL, BhagojiFLAdverserial, jimenezgutierrez2025securityprivacyfederatedlearning, ContiMITM}. These realities impose a simultaneous demand, where the system must respect both hardware limitations while still protecting the global model from unreliable or malicious updates.

Existing federated LLM fine-tuning methods address individual aspects of this problem, but they rarely treat hardware constraints and adversarial robustness as coupled design requirements. PEFT reduces the number of trainable parameters, allowing edge devices to fine-tune the LLM even in resource-constrained environments~\cite{wu2025survey, DP_fedlora}. Robust aggregation methods can mitigate the corrupted or malicious updates at the server side~\cite{xu2022byzantine, KRUM, bulyan, Median_Mean}. However, these methods are typically developed independently, where hardware conditions are treated as external runtime conditions, while robustness is treated as a server-side filtering problem over the received updates. In practice, under the edge environment, this assumption can break down. A thermally stressed client may experience throttling, delayed training, or reduced participation~\cite{LuiDVFS, HanimaiahDVFSPerformance}, while an adversarial client or communication-layer perturbation may produce abnormal updates that also appear unreliable to the server~\cite{bagdasaryanBackdoorFL, ContiMITM, BhagojiFLAdverserial}. Therefore, an edge federated fine-tuning framework must reduce the burden on resource-constrained clients while protecting the server-side aggregation from unreliable or malicious updates. This motivates the development of a framework that is hardware-aware on the client side and robust on the server-side.

To address this gap, we propose Thermo-FL, a hardware-aware and robust federated learning framework for LoRA fine-tuning of LLMs on thermally-constrained edge devices. On the client-side, Thermo-FL uses the device temperature as a control signal to regulate local LoRA adaptation and sparse update transmission. When a client becomes thermally stressed, the framework reduces computation and communication load for that client, and when the temperature remains within the safer thermal range, it can participate more fully. On the server side, Thermo-FL introduces TERRA (Thermo-Enhanced Robust Round Aggregation), a robust aggregation layer that protects the global model by applying magnitude filtering, directional validation, and adaptive clipping to received updates. Together, these client-side and server-side mechanisms allow Thermo-FL to adapt to hardware-induced irregularities while also maintaining robustness against adversarial attacks. 
To sum up, the key contributions of this paper are summarized as follows:
\begin{itemize}
    \item We propose Thermo-FL, a hardware-aware and robust framework for federated LoRA fine-tuning on thermally constrained edge devices.
    \item A temperature-driven client policy is formulated to dynamically modify LoRA training and sparse update transmission based on the device thermal conditions.
    \item TERRA is designed as a robust aggregation layer tailored to dynamically sparse LoRA updates, combining magnitude control, directional validation, and adaptive clipping to mitigate corrupted updates.
\end{itemize}

The remainder of the paper is organized as follows:  Section~\ref{sec:background} presents background on federated learning and LoRA. Section~\ref{sec:system_threat_model} defines the system and threat model. Section~\ref{sec:methodology} describes Thermo-FL and TERRA. Section~\ref{sec:eval} presents the evaluation. Section~\ref{sec:discussion} discusses implications and limitations, Section~\ref{sec:related_work} reviews related work, and Section~\ref{sec:conclusion} concludes.

\section{Preliminaries}
\label{sec:background}
In this section, we formalize the Federated Learning (FL) objective used in our system, specify how Parameter-Efficient Fine-Tuning (PEFT) is instantiated in the proposed architecture, using LoRA, and introduce the sparse-update notation used throughout this paper.

\subsection{Federated Learning Formulation}
We consider a distributed training setup consisting of a central aggregator \(\mathcal{S}\) and \(K\) edge clients. Each client \(k \in \{1, \dots, K\}\) maintains a private dataset \(\mathcal{D}_k = \{(\mathbf{x}_i, y_i)\}_{i=1}^{|\mathcal{D}_k|}\) that is never shared with the server~\cite{mcmahan2023communicationefficientlearningdeepnetworks}. The main goal of the system is to collaboratively optimize a global parameter state \(W \in \mathbb{R}^d\) that minimizes the weighted global objective
\begin{equation}
\min_{W} J(W) = \sum_{k=1}^{K} p_k F_k(W),
\qquad
\text{where }
p_k = \frac{|\mathcal{D}_k|}{\sum_{j=1}^{K} |\mathcal{D}_j|}.
\end{equation}
Here, the local loss \(F_k(W)\) denotes the empirical risk computed over \(\mathcal{D}_k\) on client \(k\), and \(p_k\) weights each client according to its local dataset size or importance.
During each communication round \(t\), the server broadcasts the current global model state \(W^{(t)}\) to all participating clients. Each client then performs local optimization, typically using stochastic gradient descent (SGD), to produce a model delta \(\Delta_k^{(t)}\). The server aggregates the received updates into \(\bar{\Delta}^{(t)}\) and updates the global model as
$
W^{(t+1)} = W^{(t)} + \bar{\Delta}^{(t)}
$, 
which is redistributed in the subsequent communication round.

\subsection{Federated LoRA Fine-Tuning}
Full-parameter fine-tuning of Large Language Models (LLMs) is often impractical on edge hardware due to memory, computation, and communication constraints.  Parameter-efficient fine-tuning (PEFT) methods reduce this burden by updating only a small set of trainable parameters instead of the full model weights~\cite{houlsby2019parameterefficienttransferlearningnlp,lester2021prompttuning,li2021prefixtuningoptimizingcontinuousprompts}. In this work, we use Low-Rank Adaptation (LoRA) as the local fine-tuning mechanism~\cite{hu2021loralowrankadaptationlarge, zhang2023lorafamemoryefficientlowrankadaptation, Kulkarni2020, loraxslowrankadaptationextremely}.
LoRA assumes that weight updates during the adaptation/training process lie in a low intrinsic subspace. For a pre-trained weight matrix $\mathbf{W}_0 \in \mathbb{R}^{d \times k}$, the updates are parameterized as a low-rank factorization $\Delta \mathbf{W} = \mathbf{BA}$, where $\mathbf{B} \in \mathbb{R}^{d \times r}$ and $\mathbf{A} \in \mathbb{R}^{r \times k}$ with rank $r \ll \min(d, k)$. In practice, this constrains adaptation to a compact subspace while preserving the representational capacity of the frozen backbone. The forward computation is therefore modified to
\begin{equation}
    \mathbf{h} = \mathbf{W}_0 \mathbf{x} + \frac{\alpha}{r} \mathbf{BA} \mathbf{x}
\end{equation}
where $\alpha$ is a scaling constant.
During local optimization, $\mathbf{W}_0$ remains fixed and only $\mathbf{A}$ and $\mathbf{B}$ receive gradient updates. 
In federated LoRA fine-tuning, each selected client \(k\) initializes from the current global LoRA state and fine-tunes the LoRA parameters on its private dataset \(\mathcal{D}_k\). After local training at round \(t\), the client produces a LoRA update \(\Delta_k^{(t)}\), which represents the change in its trainable LoRA parameters. The server aggregates the received LoRA updates into \(\bar{\Delta}^{(t)}\) and applies the result to the global LoRA state for the next communication round. This allows clients to collaboratively adapt an LLM while avoiding full-parameter training and without sharing raw local data.

\section{Hardware analysis and Threat Model}
\label{sec:system_threat_model}
In this section, we describe the hardware limitations and the adversarial framework used in this research. These constraints are important because local fine-tuning can increase device temperature and delay client updates, while adversarial clients or communication-layer perturbations can corrupt the updates used for aggregation. Therefore, our approach includes modeling two aspects of the problem: first, the hardware conditions that affect client-side training and synchronization; second, the integrity threats that affect server-side aggregation.

\subsection{Dynamic Voltage and Frequency Scaling (DVFS):}
\label{subsec:DVFS}
LoRA reduces the number of trainable parameters, but local fine-tuning still requires repeated matrix operations, optimizer updates, and memory movement on edge hardware. On the edge devices without any fan or those that have cooling constraints, sustained training can lead to thermal accumulation and runtime slowdowns~\cite{peluso2019performance, HanimaiahDVFS, GadeDVFS, HanimaiahDVFSPerformance}. Modern embedded platforms commonly rely on Dynamic Voltage and Frequency Scaling (DVFS) to reduce clock frequency when temperature approaches a critical operating region. We abstract this behavior using a simple two-state model:
\begin{equation}
    f(t) =
    \begin{cases}
    f_{\max}, & T(t) < T_{\mathrm{crit}}, \\
    f_{\mathrm{th}}, & T(t) \ge T_{\mathrm{crit}},
    \end{cases}
    \label{eq:throttle}
\end{equation}
where \(T(t)\) is the device temperature at time \(t\), \(T_{\mathrm{crit}}\) is the throttling threshold, \(f_{\max}\) is the nominal operating frequency, and \(f_{\mathrm{th}} < f_{\max}\) denotes the reduced frequency under thermal throttling.
This abstraction does not simply imply that real DVFS controllers use only two frequency states. In practice, embedded devices generally have multiple intermediate frequency levels and smooth governor policies. In Eq.~\ref{eq:throttle}, we only capture the dominant effect relevant to synchronous federated training, where once a client enters a throttled regime, its local computation time can increase substantially. In this paper, we refer to the resulting synchronization bottleneck as the thermal wall. In a synchronous FL round, the server cannot complete aggregation until the required client updates have arrived. Therefore, if client \(k\) becomes thermally throttled, its local computation time \(t_{\mathrm{comp}}^{k}\) may delay the entire round. This couples the global training throughput to the thermal state of the slowest or most constrained participating clients.

\subsection{Threat Model}
\label{subsec:threatmodel}
Federated fine-tuning of large language models on edge devices presents a significantly broader attack surface compared to the centralized training in controlled cloud environments~\cite{openprobleminFL}. Beyond the common privacy and availability of the devices concern, practical edge deployments must also contend with unreliable communication links, physically exposed hardware, and potentially adversarial participants in the federated learning rounds. In this work, we focus on integrity attacks, as they pose the most immediate threat to optimization stability, model convergence, and downstream task utility in resource-constrained federated learning systems.
We consider two integrity-oriented adversarial classes that are particularly relevant in edge federated learning: internal Byzantine clients, which manipulate updates before transmission, and external man-in-the-middle (MITM) adversaries, which tamper with updates while they are in transit. Together, these capture both endpoint-level and communication-level corruption.

\noindent \textbf{1. Internal Adversary (Byzantine Client):}
We assume that a subset of authenticated clients may deviate arbitrarily from the intended learning objective while still participating in the communication protocol~\cite{FangPosion, Async_BYzantine}. Let $\Delta_k$ denote the local model update produced by client $k$ after local training. A Byzantine client instead submits a poisoned update
\begin{equation}
    \Delta_k' = \mathcal{A}(\Delta_k),
\end{equation}
where $\mathcal{A}(\cdot)$ denotes an adversarial transformation designed to disrupt aggregation and degrade the global model.

Our threat model includes both simple and composite update-poisoning behaviors. The simplest case is a \textbf{sign-flip attack}, which reverses the update direction and amplifies its magnitude:
\begin{equation}
    A_{\mathrm{sf}}(\Delta_k) = \gamma \Delta_k, \quad \gamma < -1
\end{equation}
To model a broader and less predictable adversary, we also consider a \textbf{mixed attack} family that combines several poisoning operators studied in federated learning~\cite{FangPosion}. The purpose of this setting is to avoid evaluating robustness only against a single fixed attack pattern. Instead, each malicious client may instantiate one attack operator from a set of possible update transformations, including sign flip~\cite{Async_BYzantine}, Gaussian noise injection, scaling, random masking, model replacement~\cite{bagdasaryanBackdoorFL}, or an ALIE-like perturbation~\cite{ALIE}. Collectively, these operators span directional inversion, stochastic corruption, magnitude amplification, sparsity distortion, and distribution-aware poisoning. Their shared goal is to bias the aggregated update, destabilize convergence, and reduce the quality of the final global model.

\noindent \textbf{2. External Adversary (Man-in-the-Middle(MITM)).}
We also consider an external adversary located on the communication path between clients and the server, for example, over unsecured wireless or public network links~\cite{ContiMITM}.
Unlike a Byzantine participant, this adversary does not originate legitimate client updates, but instead intercepts the legitimate client updates and perturbs them in transit before they reach the server. We model this \textbf{MITM attack} as
\begin{equation}
    \Delta_k^{\mathrm{mitm}} = \Delta_k + z,
\end{equation}
where $z$ denotes an adversarial perturbation introduced during transmission. This formulation captures communication-layer corruption such as packet tampering, interference, or deliberate noise injection applied to otherwise benign client updates.

\subsection{System Assumptions}
The following assumptions are made throughout this paper.
(1) The central server is \textit{honest-but-curious}: it executes the protocol faithfully but could check the received updates. Fully malicious servers are not considered.
(2) Each client has a distinct identity and does not pretend to be other clients.
(3) The physical and thermal properties of all participants, even those who launch attacks, are assumed to fall within the bounds described in Section~\ref{subsec:DVFS}.
(4) The temperature readings come from hardware components on each device and thus correspond to the true state of the device.
The proposed threat model directly motivates the design of Thermo-FL. Edge deployments require defenses against adversarial updates to protect aggregation integrity, but these defenses must remain lightweight because thermal and energy budgets are constrained. Cryptographic mechanisms such as selective homomorphic encryption can protect update confidentiality, but they introduce additional system overhead and are orthogonal to our aggregation-layer robustness goal~\cite{maskcrypt}. Thermo-FL therefore uses a lightweight hardware-software co-design: client-side temperature-aware workload regulation reduces stress on constrained devices, while server-side robust aggregation limits the influence of unreliable or malicious updates.

\begin{figure*}[t]
  \centering
  \includegraphics[width=.95\textwidth]{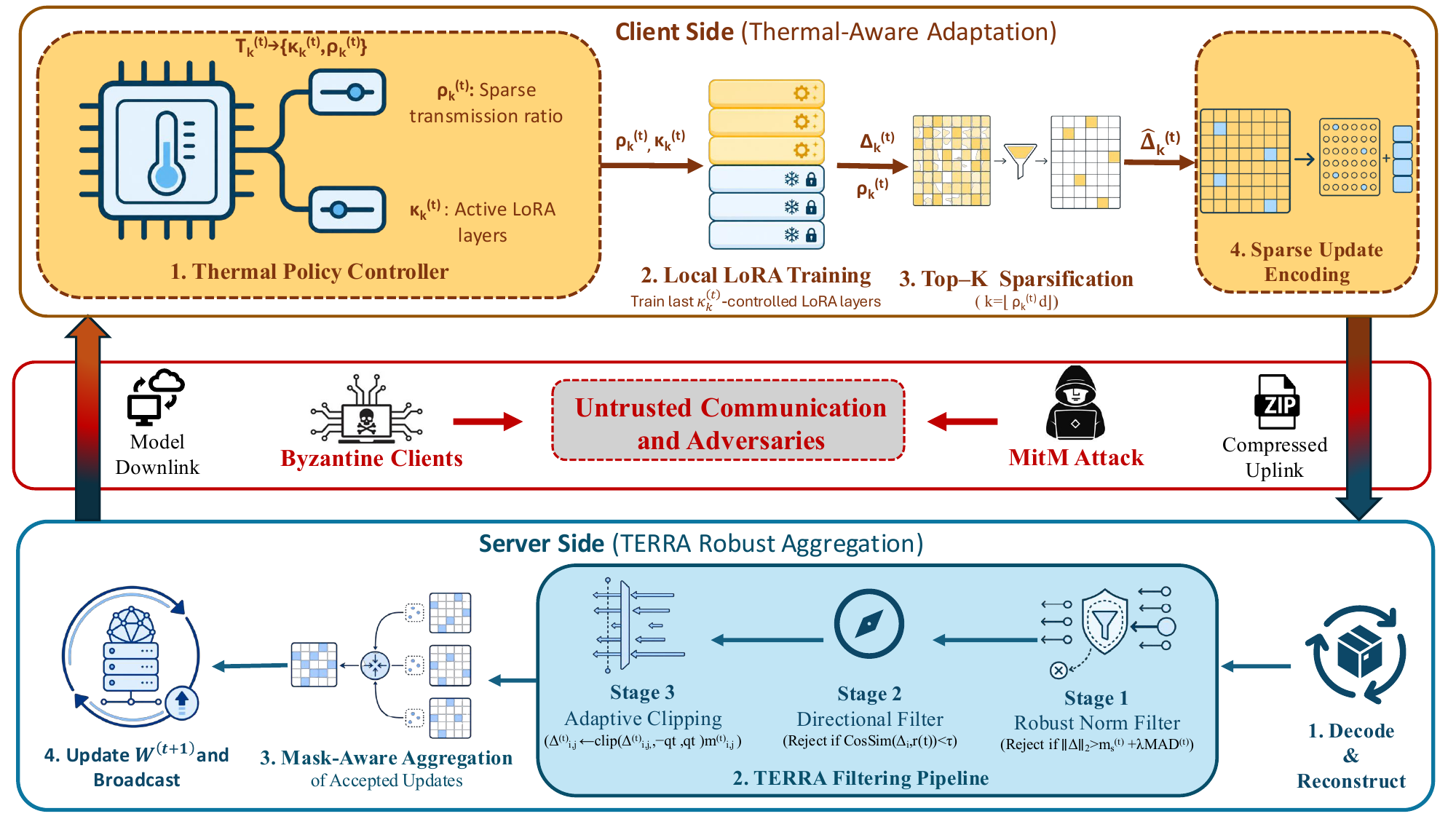}
  \caption{Overview of Thermo-FL architecture}
  \label{fig:architecture}
\end{figure*}
\section{Thermo-FL}
\label{sec:methodology}

Thermo-FL is a round-based federated fine-tuning framework architected around a key observation: in edge environments, the thermal limits, communication constraints, and adversarial updates are not independent system concerns; rather, they are tightly coupled together. Existing Federated Learning paradigms often treat hardware-side limitations as external constraints on training. In contrast, Thermo-FL converts this client's thermal state into the first-class control signal within the learning protocol itself. In each round, a client measures its local temperature and uses this signal to regulate the number of LoRA layers that remain trainable and the fraction of update coordinates transmitted to the server. The server then applies the TERRA pipeline to filter and aggregate these dynamically sparse LoRA updates, closing the loop between thermal feedback, sparse communication, and Byzantine-resilient aggregation.

\subsection{Design Overview and Control Loop}
\label{subsec:design_overview}

Fig.~\ref{fig:architecture} illustrates the end-to-end workflow of the Thermo-FL architecture across the three major components, client, communication, and server. At the start of each round, the server broadcasts the current global LoRA parameters denoted by \(W^{(t)}\) to all selected clients in the FL loop. While the pre-trained backbone of the model remains frozen, \(W^{(t)}\) represents the set of trainable LoRA adapter parameters attached to the model. On the client side, each of the selected clients measures its local temperature denoted by \(T_k^{(t)}\), then they use this measurement to determine the active LoRA-layer fraction, \(\kappa_k^{(t)}\) and sparse transmission keep ratio, \(\rho_k^{(t)}\), and perform thermally regulated local fine-tuning and computes a local LoRA update \(\Delta_k^{(t)}\), which represents the difference between the locally adapted LoRA parameters and the received global LoRA parameters. The resulting LoRA update obtained after the difference is then sparsified, encoded, and transmitted to the server. On the server side, we use the TERRA pipeline to screen out any abnormal updates, validate their direction, clip retained coordinates, and perform mask-aware aggregation to produce the global aggregate update denoted by \(\bar{\Delta}^{(t)}\). The server applies this aggregate update to obtain \(W^{(t+1)}\), which is then used to begin the next communication round.
The round-level control loop can be summarized as
\begin{equation}
\label{eq:thermofl_control_loop}
\begin{aligned}
W^{(t)}
&\rightarrow
\{T_k^{(t)}\}_{k\in S_t}
\rightarrow
\{(\kappa_k^{(t)},\rho_k^{(t)})\}_{k\in S_t} \\
&\rightarrow
\{\Delta_k^{(t)}\}_{k\in S_t}
\rightarrow
\{z_k^{(t)}\}_{k\in S_t}
\rightarrow
U^{(t)} \\
&\rightarrow
\mathcal{F}_{\mathrm{TERRA}}(U^{(t)})
\rightarrow
\bar{\Delta}^{(t)}
\rightarrow
W^{(t+1)} .
\end{aligned}
\end{equation}
where, \(S_t\) denotes the set of clients selected in communication round \(t\), and \(T_k^{(t)}\) is the measured device temperature of client \(k\). The two temperature-controlled parameters play different roles: \(\kappa_k^{(t)}\in(0,1]\) controls the fraction of LoRA layers that remain trainable during local fine-tuning, whereas \(\rho_k^{(t)}\in(0,1]\) controls the fraction of update coordinates retained for transmission. Each client computes a local LoRA update difference \(\Delta_k^{(t)}\) relative to the received global LoRA parameters, sparsifies and encodes this difference into a transmitted payload \(z_k^{(t)}\), and sends the payload to the server. After decoding, the server obtains the sparse update set \(U^{(t)}=\{u_k^{(t)}:k\in S_t\}\). TERRA, denoted by \(\mathcal{F}_{\mathrm{TERRA}}(\cdot)\), then operates on this decoded sparse update set and produces the aggregate update \(\bar{\Delta}^{(t)}\). The server then applies this aggregate difference to the current global LoRA parameters, yielding the next round global LoRA weight which is denoted by\(W^{(t+1)}=W^{(t)}+\bar{\Delta}^{(t)}\).

The two client-side control parameters serve different roles. The active layer fraction \(\kappa_k^{(t)}\) reduces local backpropagation workload by limiting how many LoRA layers are updated on client \(k\). The transmission keep ratio \(\rho_k^{(t)}\) reduces uplink payload size by limiting how many update coordinates are retained after local training. Both parameters are derived from the measured thermal state, but they act on different stages of the round: \(\kappa_k^{(t)}\) affects local optimization, whereas \(\rho_k^{(t)}\) affects communication.
This design creates a negative feedback loop between device temperature and learning workload. When a client heats up, the protocol lowers the amount of local computation and communication assigned to that client. When the client operates below the thermal threshold, it can participate with a larger active LoRA subset and a denser transmitted update. Thermo-FL does not require an exact thermodynamic model of each device. Instead, it uses measured temperature as a practical feedback signal for adapting the training and transmission behavior of heterogeneous edge clients.

The following subsections describe the two main components of this control loop. Section~\ref{subsec:client_policy} presents the client-side thermal policy, including the selection of trainable LoRA layers and the temperature-driven sparse transmission mechanism. Section~\ref{subsec:TERRA} presents TERRA, the server-side robust aggregation pipeline for filtering and aggregating dynamically sparse LoRA updates.

\subsection{Client-Side Thermal-Adaptive Training and Sparse Transmission}
\label{subsec:client_policy}
The Thermo-FL architecture's client-side component has two coupled mechanisms. First of all, each client measures its device temperature, which determines the fraction of LoRA layers that remain trainable during local fine-tuning. Second, the same thermal signal determines the fraction of update coordinates retained for transmission after training. We describe these mechanisms in the following subsections.

\noindent
\paragraph{Thermal-adaptive LoRA training}
Here, Let \(T_k^{(t)}\) denote the recorded temperature of client \(k\) at round \(t\). Thermo-FL then maps this temperature to an active LoRA-layer fraction denoted by \(\kappa_k^{(t)}\), which controls the portion of the LoRA layers that participate in backpropagation:
\begin{equation}
\label{eq:kappa_policy}
\kappa_k^{(t)} =
\begin{cases}
1.0, & T_k^{(t)} < T_{\mathrm{low}},\\
\kappa_{\mathrm{mid}}, & T_{\mathrm{low}} \leq T_k^{(t)} \leq T_{\mathrm{high}},\\
\kappa_{\mathrm{crit}}, & T_k^{(t)} > T_{\mathrm{high}},
\end{cases}
\end{equation}
where \(0 < \kappa_{\mathrm{crit}} \leq \kappa_{\mathrm{mid}} < 1\) are throttling constants. Only the last \(\kappa_k^{(t)}\) fraction of LoRA layers remains trainable during local fine-tuning. When the device operates below the lower thermal threshold \(T_{\mathrm{low}}\), all LoRA layers are trained. In the intermediate thermal region, the Thermo-FL framework freezes a portion of the adapter and updates only the last \(\kappa_{\mathrm{mid}}\) fraction of LoRA layers. Under high thermal stress, training is restricted to the last \(\kappa_{\mathrm{crit}}\) fraction. This discrete control policy avoids frequent changes in the trainable layer set while still reducing the local backpropagation workload as the device approaches a higher operating temperature.

Let \(M^{(t)}_k\) denote the layer-freezing mask induced by \(\kappa^{(t)}_k\). 
For notational simplicity, we write the masked local optimization in one-step form; in implementation, the same mask is applied throughout the configured local optimizer steps. 
The local LoRA update is then restricted to the active layers:
\begin{equation}
W^{(t)}_{k,\mathrm{local}} =
W^{(t)} - \eta M^{(t)}_k \nabla F_k(W^{(t)}),
\end{equation}
%
where, \(\eta\) is the learning rate and \(F_k\) is the local objective of client \(k\). The mask \(M_k^{(t)}\) applies only to the trainable LoRA adapter parameters, while the pre-trained backbone model weight remains frozen. After local fine-tuning, the client constructs the LoRA update difference as:
$
\label{eq:client_delta}
\Delta_k^{(t)}
=
W_{k,\mathrm{local}}^{(t)}
-
W^{(t)} .
$
Thus, \(\Delta_k^{(t)}\) captures the effect of thermally regulated local optimization during round \(t\). Thermo-FL transmits a sparsified representation of this update difference rather than the full locally adapted LoRA parameters.
\paragraph{Temperature-driven sparse transmission}
After local fine-tuning, Thermo-FL applies sparse transmission to the LoRA update difference \(\Delta_k^{(t)}\) obtained from the thermal-adaptive LoRA training. This stage is controlled by the sparse transmission ratio, also called the keep ratio \(\rho_k^{(t)}\), which determines the fraction of update coordinates retained for upload. Unlike \(M_k^{(t)}\), which is a layer-freezing mask used during local optimization, the sparse transmission mask is denoted by \(m_k^{(t)}\) and operates over update coordinates after the fine-tuning takes place.
Thermo-FL maps the recorded temperature, $T_k^{(t)}$  to the transmission keep ratio as
\begin{equation}
\label{eq:rho_policy}
\rho_k^{(t)}
=
\operatorname{clip}
\left(
b - \omega T_k^{(t)},
\rho_{\min},
\rho_{\max}
\right),
\qquad \omega > 0 ,
\end{equation}
where \(0 < \rho_{\min} \leq \rho_{\max} \leq 1\). This policy makes the retained update fraction decrease with an increase in temperature: cooler clients may transmit denser updates, while thermally stressed clients transmit fewer coordinates. Given \(d\) LoRA update coordinates, the client retains
$
\label{eq:topk_size}
s_k^{(t)}
=
\left\lfloor \rho_k^{(t)} d \right\rfloor
$
coordinates with the largest magnitudes in \(\Delta_k^{(t)}\). The resulting support mask is
\begin{equation}
\label{eq:sparse_mask}
m_{k,j}^{(t)}
=
\mathbbm{1}
\left[
j \in \operatorname{TopK}
\left(
|\Delta_k^{(t)}|,
s_k^{(t)}
\right)
\right],
\end{equation}
where \(m_{k,j}^{(t)}=1\) denotes that the coordinate \(j\) is retained for transmission. The sparse update before encoding can therefore be given as:
$
\label{eq:sparse_delta}
\tilde{\Delta}_k^{(t)}
=
m_k^{(t)} \odot \Delta_k^{(t)} .
$
The client then encodes its sparse update as a compact payload
$
\label{eq:encoded_sparse_payload}
z_k^{(t)}
=
\operatorname{Encode}
\left(
\tilde{\Delta}_k^{(t)}, m_k^{(t)}
\right),
$
where the support \(m_k^{(t)}\) is represented using a bitmap and the retained values are serialized and compressed before upload. In our implementation, the serialized sparse payload is compressed using zlib. This representation is well-suited to dynamically sparse LoRA updates because it preserves the coordinate support needed by the server while avoiding transmission of the full dense update vector, thus reducing the overall payload from the edge devices.
On the server side, it receives the encoded payloads \(\{z_k^{(t)}\}_{k\in S_t}\), decodes them into sparse update-mask pairs, and applies TERRA to aggregate the reconstructed updates robustly, as described in the following section.
\subsection{Server-Side Robust Aggregation (TERRA)}
\label{subsec:TERRA}
After receiving the encoded sparse payloads, the server then decodes each payload into a reconstructed sparse update, $\tilde{\Delta}_k^{(t)}$ and its support mask, $m_k^{(t)}$. We denote the decoded update from client \(k\) as
$
u_k^{(t)}
=
(\tilde{\Delta}_k^{(t)}, m_k^{(t)}),
$
where \(\tilde{\Delta}_k^{(t)}\) is the reconstructed sparse LoRA update and \(m_k^{(t)}\in\{0,1\}^d\) is the corresponding binary support mask. We denote the decoded update set used by TERRA as
$
U^{(t)}
=
\{u_k^{(t)}:k\in S_t\}.
$
For norm and direction checks, we flatten each sparse update into a single vector across all LoRA tensors. The support mask is retained for coordinate-wise clipping and mask-aware aggregation.

TERRA follows the broader robust-aggregation principle of screening client updates before aggregation. Prior Byzantine-robust FL work has shown that magnitude, sign, and similarity statistics can help identify malicious gradients before averaging~\cite{xu2022byzantine}. We adapt this principle to dynamically sparse LoRA updates by applying norm filtering and directional validation to decoded sparse updates while preserving their support masks in the TERRA pipeline. For clipping, we follow the motivation behind adaptive clipping methods, where clipping thresholds are derived from the update distribution rather than fixed a priori~\cite{thakkar2019differentially}. Unlike DP-focused adaptive clipping, TERRA uses clipping as a robustness mechanism for sparse LoRA aggregation and applies it only over active coordinates.
We express the pipeline as
\begin{equation}
\label{eq:terra_pipeline}
\mathcal{F}_{\mathrm{TERRA}}
=
\mathcal{F}_{\mathrm{agg}}
\circ
\mathcal{F}_{\mathrm{clip}}
\circ
\mathcal{F}_{\mathrm{dir}}
\circ
\mathcal{F}_{\mathrm{norm}} .
\end{equation}
where, \(\mathcal{F}_{\mathrm{norm}}\) performs robust norm filtering, \(\mathcal{F}_{\mathrm{dir}}\) performs directional validation, \(\mathcal{F}_{\mathrm{clip}}\) applies adaptive coordinate-wise clipping on active coordinates, and \(\mathcal{F}_{\mathrm{agg}}\) performs mask-aware aggregation. The output of this pipeline is the aggregate update
$
\label{eq:terra_output}
\bar{\Delta}^{(t)}
=
\mathcal{F}_{\mathrm{TERRA}}(U^{(t)}),
$
which is applied to the global LoRA parameters as \(W^{(t+1)}=W^{(t)}+\bar{\Delta}^{(t)}\).

\paragraph{Robust norm filtering}
The first stage of the TERRA pipeline, denoted by \(\mathcal{F}_{\mathrm{norm}}\), enforces a round-adaptive magnitude constraint on the decoded sparse updates. For each reconstructed update \(u_k^{(t)}=(\tilde{\Delta}_k^{(t)},m_k^{(t)})\), the server computes the \(\ell_2\)-norm of the sparse LoRA update:
$
s_k^{(t)}
=
\|\tilde{\Delta}_k^{(t)}\|_2 .
$
Let
$
m_s^{(t)}
=
\operatorname{median}
\left(
\{s_k^{(t)}:k\in S_t\}
\right)
$
denote the median update norm in round \(t\), and let
$
\operatorname{MAD}^{(t)}
=
\operatorname{median}
\left(
\{|s_k^{(t)}-m_s^{(t)}|:k\in S_t\}
\right)
$
denote the corresponding median absolute deviation. We have designed TERRA to retain client \(k\) after norm filtering only if
\begin{equation}
\label{eq:terra_norm_filter}
s_k^{(t)}
\leq
m_s^{(t)}
+
\lambda_{\mathrm{norm}}\operatorname{MAD}^{(t)},
\end{equation}
where \(\lambda_{\mathrm{norm}}\) controls the tolerance of the magnitude filter. The resulting norm-valid client set is
\begin{equation}
\label{eq:terra_norm_valid_set}
S_{\mathrm{norm}}^{(t)}
=
\left\{
k\in S_t:
s_k^{(t)}
\leq
m_s^{(t)}
+
\lambda_{\mathrm{norm}}\operatorname{MAD}^{(t)}
\right\}.
\end{equation}
This stage is designed to remove updates whose magnitudes are unusually large relative to the other sparse LoRA updates received in the same round by the server. Because the threshold is derived from the round-level norm distribution, it adapts to changes in update scale across training rather than relying on a fixed global cutoff. This is particularly effective against scaling and model-replacement attacks, where an adversarial client attempts to dominate the aggregate by submitting an update with excessive norm.

\paragraph{Mask-aware directional validation}
The second stage of TERRA, denoted by \(\mathcal{F}_{\mathrm{dir}}\), applies directional consistency with the recent global update trajectory. The first stage, the norm filter, removes extreme-magnitude updates, but it does not guarantee that all retained updates point in a useful direction. An adversarial client may still submit an update with a plausible magnitude but harmful orientation. To address this case, using the second stage of TERRA, we compare each norm-valid update against a server-maintained reference direction \(r^{(t)}\), while respecting the sparse support transmitted by that client.
The server maintains the reference direction as an exponential moving average of previously accepted aggregate updates:
\begin{equation}
\label{eq:terra_reference_direction}
r^{(t)}
=
\beta r^{(t-1)}
+
(1-\beta)\bar{\Delta}^{(t-1)},
\end{equation}
where \(\beta\in[0,1)\) is the momentum coefficient that controls the memory of the reference direction, and \(\bar{\Delta}^{(t-1)}\) is the aggregate update applied in the previous round. For each \(k\in S_{\mathrm{norm}}^{(t)}\), TERRA computes cosine similarity only over the coordinates transmitted by that client. Specifically, for each \(k\in S_{\mathrm{norm}}^{(t)}\), TERRA computes the mask-aware cosine score \(c_k^{(t)}\) by projecting the reference direction onto the support mask \(m_k^{(t)}\):
\begin{equation}
\label{eq:cosine_sim}
c_k^{(t)}
=
\frac{
\left\langle
\tilde{\Delta}_k^{(t)},
m_k^{(t)} \odot r^{(t)}
\right\rangle
}{
\|\tilde{\Delta}_k^{(t)}\|_2
\,
\|m_k^{(t)} \odot r^{(t)}\|_2
+
\epsilon
}.
\end{equation}

Here, \(\epsilon>0\) prevents division by zero. This score measures the alignment between the reconstructed sparse update and the recent aggregate trajectory over only the coordinates transmitted by client \(k\) which avoids penalizing a sparse client for coordinates it did not transmit and that should not be interpreted as explicit zeros. TERRA retains client \(k\) after directional validation only if
\begin{equation}
\label{eq:terra_direction_valid_set}
S_{\mathrm{dir}}^{(t)}
=
\left\{
k\in S_{\mathrm{norm}}^{(t)}:
c_k^{(t)} \geq \tau_{\mathrm{dir}}
\right\}.
\end{equation}
where, \(\tau_{\mathrm{dir}}\) is the directional similarity threshold. Low or negative alignment indicates that the update is inconsistent with the recent optimization trajectory and may correspond to sign-flipping or poisoning behavior. For the first round, when no historical reference is available, TERRA initializes \(r^{(t)}\) from the aggregate of norm-valid updates or skips directional validation for that round.

\paragraph{Adaptive active-coordinate clipping}
The third stage of TERRA, denoted by \(\mathcal{F}_{\mathrm{clip}}\), bounds the coordinate-level influence of updates that passes through norm filtering and directional validation. Let \(S_{\mathrm{dir}}^{(t)}\) denote the set of clients retained after the first two stages.  For a reconstructed sparse update \(\tilde{\Delta}_k^{(t)}\), we write \(\tilde{\Delta}_{k,j}^{(t)}\) for its \(j\)-th coordinate and \(m_{k,j}^{(t)}\) for the corresponding support indicator. Rather than using a fixed clipping radius, TERRA first collects the magnitudes of all active coordinates among the retained sparse updates which is denoted by $\mathcal{C}^{(t)}$ and defined as
\begin{equation}
\mathcal{C}^{(t)}
=
\left\{
\left|\tilde{\Delta}_{k,j}^{(t)}\right|
:
k\in S_{\mathrm{dir}}^{(t)},\;
m_{k,j}^{(t)}=1
\right\}.
\end{equation}
The clipping threshold is then defined as \(q_t=\operatorname{Quantile}_{q}(\mathcal{C}^{(t)})\), where \(q\) is a fixed quantile hyperparameter. For each retained client \(k\in S_{\mathrm{dir}}^{(t)}\), this stage of TERRA clips only the active coordinates as

\begin{equation}
\label{eq:terra_active_coordinate_clipping}
\hat{\Delta}_{k,j}^{(t)}
=
m_{k,j}^{(t)}
\operatorname{clip}
\left(
\tilde{\Delta}_{k,j}^{(t)},
-q_t,
q_t
\right).
\end{equation}
Let \(\hat{\Delta}_k^{(t)}\) denote the clipped sparse update vector obtained after the three stages of TERRA filtering. This operation limits the influence of extreme individual coordinates while preserving the sparse support structure of each client update. Since the threshold is computed only from active coordinates, TERRA avoids mixing transmitted values with missing coordinates that were never uploaded by the client.

\paragraph{Mask-aware aggregation}
The final stage, \(\mathcal{F}_{\mathrm{agg}}\), is responsible for aggregating the clipped sparse updates without treating missing coordinates as zeros. For each coordinate \(j\), we compute
\begin{equation}
\label{eq:terra_mask_aware_aggregation}
\bar{\Delta}_j^{(t)}
=
\frac{
\sum_{k\in S_{\mathrm{dir}}^{(t)}}
w_k\, m_{k,j}^{(t)} \hat{\Delta}_{k,j}^{(t)}
}{
\sum_{k\in S_{\mathrm{dir}}^{(t)}}
w_k\, m_{k,j}^{(t)}
+
\epsilon
},
\end{equation}
where \(w_k\) is the aggregation weight of client \(k\), typically proportional to its local sample count, and \(\epsilon>0\) prevents division by zero when no retained client transmits coordinate \(j\). This mask-aware form is important in the sparse setting because an absent coordinate indicates missing support rather than evidence that the client intended to submit a zero value. The server then applies the aggregate update to the global LoRA parameters as
$
W^{(t+1)}
=
W^{(t)}
+
\bar{\Delta}^{(t)}.
$
Finally, the reference direction used for the next round is updated as
$
r^{(t+1)}
=
\beta r^{(t)}
+
(1-\beta)\bar{\Delta}^{(t)},
$
where \(\beta\in[0,1)\) is the momentum coefficient. To sum it up, TERRA combines 3 stages of norm filtering, mask-aware directional validation, and adaptive active-coordinate clipping, followed by mask-aware aggregation to limit abnormal update influence while preserving useful sparse learning signal from retained clients.

\subsection{Robustness Analysis}
\label{subsec:security_rationale}


In this section, we establish a design-level robustness rationale for TERRA pipeline under the threat model introduced in the Section~\ref{subsec:threatmodel}. Rather than treating the pipeline as a heuristic sequence of the filters, we show how each of the stages of the TERRA impose the concrete constraint on the adversarial update space. Robust norm filtering limits excessive update magnitude, mask-aware directional validation rejects updates that are misaligned with the recent aggregate trajectory, adaptive active-coordinate clipping bounds the coordinate-level influence of retained updates, and mask-aware aggregation prevents missing sparse coordinates from being interpreted as explicit zeros. We do not intended to overview this analysis as a full convergence or cryptographic security proof; instead, it formalizes the bounded-influence behavior that TERRA provides for dynamically sparse LoRA aggregation.

\paragraph{Magnitude-bounded update influence}
Let us consider \(A_t\subseteq S_t\) be the adversarial clients selected in round \(t\), and let \(H_t=S_t\setminus A_t\) denote the honest clients in the same round \(t\). The first TERRA stage retains only clients in \(S_{\mathrm{norm}}^{(t)}\). By construction, any retained update satisfies the condition:
$
\|\tilde{\Delta}_k^{(t)}\|_2
\leq
B_t, 
$
where
$
B_t
=
m_s^{(t)}
+
\lambda_{\mathrm{norm}}\operatorname{MAD}^{(t)} .
$
Thus, any adversarial update whose norm exceeds the round-adaptive bound \(B_t\) is rejected before aggregation. For example, under a scaling or model-replacement attack where an adversary submits \(\tilde{\Delta}_a^{(t)}=\gamma \Delta_a^{(t)}\), the update is removed whenever
\begin{equation}
        |\gamma|\,\|\Delta_a^{(t)}\|_2 > B_t ,
\end{equation}

This does not imply that all malicious updates are eliminated. Rather, it prevents high-energy adversarial updates from directly dominating the aggregate through excessive norm, provided the round-level robust statistics are not themselves dominated by adversarial clients.

\paragraph{Directionally inconsistent updates}
An adversarial update which passes the norm filter may still pointing in a harmful direction. TERRA therefore applies mask-aware directional validation over the support transmitted by each client. For each \(k\in S_{\mathrm{norm}}^{(t)}\), the alignment score as given in the eq.~\ref{eq:cosine_sim} is 
\[
c_k^{(t)}
=
\frac{
\left\langle
\tilde{\Delta}_k^{(t)},
m_k^{(t)}\odot r^{(t)}
\right\rangle
}{
\|\tilde{\Delta}_k^{(t)}\|_2
\,
\|m_k^{(t)}\odot r^{(t)}\|_2
+
\epsilon
}.
\]
The update is retained if and only if \(c_k^{(t)}\geq \tau_{\mathrm{dir}}\). This filter is scale invariant and thus complements norm filtering, as it does not require all retained updates to have similar magnitudes, but it requires the transmitted coordinates to align with the recent accepted update trajectory. 
To understand the effect of this layer on sign-flip attacks, let us consider an honest update \(\Delta_h^{(t)}\) with positive alignment score \(c_h^{(t)}\). If an adversary submits a sign-flipped version \(\tilde{\Delta}_a^{(t)}=-\gamma \Delta_h^{(t)}\) with \(\gamma>0\), then its alignment score \(c_a^{(t)}\) reverses sign on the same active support, giving \(c_a^{(t)}\approx -c_h^{(t)}\). Such updates are rejected whenever \(c_a^{(t)}<\tau_{\mathrm{dir}}\). Thus, mask-aware directional validation of TERRA constrains all of the attacks that preserve plausible magnitude but distort the update direction.

\paragraph{Coordinate-level bounded influence}
Even after the first two filtering stages of norm and direction filtering, a retained adversarial update may place large values in a small number of coordinates. We limit this effect through the third stage of TERRA, adaptive active-coordinate clipping. For each retained client \(k\in S_{\mathrm{dir}}^{(t)}\), the clipped coordinate satisfies the condition
$
\hat{\Delta}_{k,j}^{(t)}
=
m_{k,j}^{(t)}
\operatorname{clip}
\left(
\tilde{\Delta}_{k,j}^{(t)},
-q_t,
q_t
\right),
|\hat{\Delta}_{k,j}^{(t)}|
\leq
m_{k,j}^{(t)}q_t .
$
Therefore, any update that survives the first two stages has bounded coordinate-level influence. If each client transmits at most \(\rho_{\max}d\) coordinates, then the clipped sparse update also satisfies
$
\|\hat{\Delta}_k^{(t)}\|_2
\leq
q_t\sqrt{\rho_{\max}d}.
$
This bound is useful against attacks that concentrate perturbations into a small set of coordinates. The clipping stage does not certify that a retained update is honest; it ensures that the coordinate-wise contribution of any retained update is bounded before the aggregation.

\paragraph{Bounded adversarial contribution after aggregation}
The final stage of TERRA is the mask aware aggregation. For coordinate \(j\), TERRA aggregates only over clients that transmitted that coordinate, which is give in eq.~\ref{eq:terra_mask_aware_aggregation} is:
\[
\bar{\Delta}_j^{(t)}
=
\frac{
\sum_{k\in S_{\mathrm{dir}}^{(t)}}
w_k m_{k,j}^{(t)}\hat{\Delta}_{k,j}^{(t)}
}{
\sum_{k\in S_{\mathrm{dir}}^{(t)}}
w_k m_{k,j}^{(t)}
+
\epsilon
}.
\]
Let
\(
D_j^{(t)}
=
\sum_{k\in S_{\mathrm{dir}}^{(t)}}
w_k m_{k,j}^{(t)}
+
\epsilon
\)
be the active aggregation mass for coordinate \(j\), and let
\(
A_j^{(t)}
=
\{a\in A_t\cap S_{\mathrm{dir}}^{(t)}:m_{a,j}^{(t)}=1\}
\)
be the retained adversarial clients that transmitted coordinate \(j\). As described in above subsection, clipping gives \(|\hat{\Delta}_{a,j}^{(t)}|\leq q_t\), the adversarial contribution to coordinate \(j\) is bounded by
\[
\left|
\frac{
\sum_{a\in A_j^{(t)}}
w_a m_{a,j}^{(t)}\hat{\Delta}_{a,j}^{(t)}
}{
D_j^{(t)}
}
\right|
\leq
\frac{
\sum_{a\in A_j^{(t)}} w_a
}{
D_j^{(t)}
}
q_t .
\]
Thus, any retained adversarial influence is controlled by both the clipping threshold \(q_t\) as well as the adversarial active aggregation weight on that coordinate. This is simply a bounded-influence statement, not a claim that every adversarial update is removed.
\paragraph{Sparse-support manipulation}
Sparse LoRA updates introduce an additional attack surface because clients may transmit different coordinate supports based on the sparse transmission ratio denoted by, $\rho_k^{(t)}$. If all of the missing coordinates were treated as zeros, coordinates transmitted by fewer clients would have been artificially shrunken toward zero, and thus support-manipulation attacks could distort the aggregate. But, TERRA avoids this kind of manipulation by normalizing each coordinate only over the clients that actually transmit that coordinate. In other words, a missing coordinate is treated as missing support, not as an explicit zero-valued updates. This distinction is essential for dynamically sparse updates, where the support pattern is shaped by top-\(k\) transmission and client thermal state.

To sum it up, we have designed the TERRA pipeline to limit the adversarial influence through a sequence of complementary constraints. First of all, Norm filtering removes extreme-magnitude updates, followed by mask-aware directional validation reducing sign-flipped or inconsistent updates. Then, adaptive clipping bounds the coordinate-level effect of retained updates, and finally, mask-aware aggregation prevents sparse missingness from being interpreted as zero evidence. These mechanisms do not replace cryptographic guarantees or fully eliminate Byzantine behavior, but they provide a robust aggregation layer tailored to the sparse LoRA updates produced by Thermo-FL. Appendix~\ref{app:convergence} further provides a convergence rationale under an abstract TERRA operator, showing how bounded influence, alignment, and second-moment control lead to a stationary-point bound under explicit assumptions.

\section{Evaluation}
\label{sec:eval}
We evaluated Thermo-FL framework under two different yet complementary environments. In the first environment, the large-scale emulator was utilized to investigate the robustness of the controlled federated learning in the presence of large number of clients, non-IID distribution of the data, LoRA transmissions in their dense or sparse form, and under adversarial updates. On the other hand, the physical testbed was employed to verify the effectiveness of the entire pipeline under real-world limitations in terms of temperature, runtime variation, compressed payload size, and defense against attacks. Together, these settings separated controlled robustness analysis from physical deployment feasibility.

\noindent\textbf{Evaluation goals:}
Our evaluation was organized around the following research questions:

\noindent\textbf{RQ1: Robustness.}
Does the TERRA pipeline improve the robustness against sign flip and mixed Byzantine attacks when compared with dense and sparse FedAvg-LoRA, and other robust aggregation baselines?

\noindent\textbf{RQ2: Sparse aggregation.}
How is task utility affected when LoRA updates are transmitted and aggregated in dynamically sparse form rather than as full-density LoRA deltas?

\noindent\textbf{RQ3: TERRA components.}
Which stages of TERRA are necessary for robust sparse aggregation, and how much do norm filtering, directional validation, adaptive clipping, and mask-aware aggregation contribute to the final behavior?

\noindent\textbf{RQ4: Physical feasibility.}
Can Thermo-FL be executed end-to-end on real edge devices while capturing thermal behavior, runtime overhead, communication cost, and attack response?
\begin{table*}[t]
\centering
\caption{Final exact-match accuracy (\%) at global round 25 in the emulator using Qwen2.5-0.5B. Higher is better.}
\label{tab:round25_emulator_qwen05b}
\small
\setlength{\tabcolsep}{12pt}
\begin{tabular}{lcccccc}
\toprule
& \multicolumn{3}{c}{GSM8K} & \multicolumn{3}{c}{BoolQ} \\
\cmidrule(lr){2-4} \cmidrule(lr){5-7}
Method & Clean & Signflip & Mixed & Clean & Signflip & Mixed \\
\midrule
Qwen-ZS                & 27.44 & --    & --    & 62.88 & --    & --    \\
\midrule
FedAvg (dense)              & 33.21 & 24.64 & 24.56 & 69.30 & 63.39 & 58.10 \\
Trimmed Mean (dense)        & 32.52 & 28.96 & 32.30 & 69.08 & 68.59 & 59.54 \\
Coordinate Median (dense)   & 30.93 & 28.66 & \textbf{33.06} & 68.75 & 68.59 & 64.92 \\
Krum (dense)                & 24.11 & 10.16 & 20.92 & 68.83 & 66.79 & 60.55 \\
Multi-Krum (dense)          & \textbf{34.27} & 30.86 & 31.39 & 67.06 & 69.24 & 69.85 \\
Bulyan (dense)              & 31.61 & 28.81 & 30.71 & 68.90 & 68.38 & 60.46 \\
\midrule
FedAvg (sparse)            & 32.90 & 26.69 & 29.57 & 69.39 & 63.46 & 62.63 \\
Trimmed Mean (sparse)      & 32.75 & 31.39 & 28.58 & 68.99 & 68.13 & 67.58 \\
Coordinate Median (sparse) & 34.19 & 30.10 & 30.71 & 69.36 & 66.91 & 66.15 \\
Krum (sparse)              & 30.43 & 18.57 & 30.63 & 68.20 & 67.25 & 60.06 \\
Multi-Krum (sparse)        & 32.22 & \textbf{32.83} & 31.39 & 69.24 & 69.11 & 68.06 \\
Bulyan (sparse)            & 30.93 & 30.93 & 28.81 & 69.27 & 67.92 & 67.80 \\
Thermo-FL (TERRA)          & 33.74 & 31.09 & 31.67 & \textbf{72.32} & \textbf{71.16} & \textbf{72.11} \\
\bottomrule
\end{tabular}
\end{table*}
\subsection{Large-Scale Emulator}
\label{subsec:large_scale_emulator}
A large-scale emulator was used to evaluate the performance of Thermo-FL in controlled federated learning environments that are difficult to emulate physically. This setup enabled the analysis of larger client population sizes, non-IID data distribution, sparse LoRA communication, malicious updates, and temperature-aware client control under repeatable conditions.

\subsubsection{Experimental Setup}
\label{subsubsec:emulator_setup}
We configured the emulator experiment to ensure the model, federated training approach, sparsity technique, attack configuration, and evaluation process was consistent across all methods.

\noindent\textbf{Implementation and tasks.}
The implementation of the emulator involved using PyTorch, HuggingFace transformers, PEFT/LoRA, and torch.distributed. Qwen2.5-0.5B~\cite{qwen2025qwen25technicalreport} was used as the base model for this study, and it was tested on GSM8K~\cite{cobbe2021trainingverifierssolvemath} and BoolQ~\cite{clark-etal-2019-boolq} datasets. GSM8K measured numerical reasoning and arithmetic answer extraction ability, whereas BoolQ tested binary question-answer ability. The exact match was chosen as the utility metric.

\noindent\textbf{Federated protocol.}
Training data were partitioned using a non-IID Dirichlet split with concentration parameter \(\alpha=0.3\). For the round-25 experiments in Table~\ref{tab:round25_emulator_qwen05b}, we emulated 80 logical clients and sampled 10 clients per round, corresponding to a 12.5\% participation rate. Each selected client performed one local epoch with 32 optimizer steps, batch size 1, and gradient accumulation of 4. We used AdamW with learning rate \(2\times10^{-4}\) and weight decay \(10^{-2}\).

\noindent\textbf{LoRA and sparsity.}
All of the methods communicated LoRA adapter deltas rather than full model weights. The LoRA configuration used rank \(r=16\), scaling parameter \(\alpha=32\), dropout 0.05, automatic target-module selection, FP16 training, gradient checkpointing, and maximum sequence length 384. Dense baselines transmitted full-density LoRA deltas, whereas sparse methods applied top-\(k\) sparsification before transmission.

\noindent\textbf{Baselines and TERRA.}
Thermo-FL was compared against dense and sparse variants of FedAvg-LoRA, trimmed mean, coordinate-wise median, Krum, Multi-Krum, and Bulyan. Dense variants aggregated full-density LoRA adapter deltas, whereas sparse variants aggregated top-\(k\) sparsified LoRA adapter deltas. Thermo-FL used TERRA for sparse LoRA aggregation. Unless otherwise stated, TERRA used a MAD scaling factor of 2.5, cosine threshold \(-0.05\), warmup period of 2 rounds, clipping quantile 0.80, and FedAvg-style reduction after filtering.

\noindent\textbf{Thermal model.}
Each logical client maintained a persistent temperature state across rounds. Temperatures were initialized from \(\mathcal{N}(44,1.5^2)\), increased during local optimization, and cooled during idle periods. The controller thresholds were set to \(T_{\min}=45^\circ\mathrm{C}\) and \(T_{\max}=65^\circ\mathrm{C}\). The emulator used a cooling factor of 1.25 per round and a heating factor of 0.65 per local round. Temperature controlled both the active trainable-layer ratio \(\kappa \in [0.20,1.00]\) and the sparse transmission ratio \(\rho \in [0.05,0.50]\), so hotter clients trained fewer LoRA layers and transmitted sparser updates.

\noindent\textbf{Adversarial settings.}
We evaluated clean, sign-flip, and mixed attack settings. In adversarial runs, 20\% of clients were malicious, giving 16 malicious clients out of 80 and approximately two malicious clients among the 10 sampled clients per round. In the sign-flip setting, malicious clients reversed their update direction along with scaling \(\gamma=-3\). In the mixed setting, attacks were sampled from sign flip, Gaussian noise, scale attack with \(\gamma=-3\), random-mask corruption, model replacement, and ALIE-like perturbations. Attacks were applied after local training and client-side sparsification.

\noindent\textbf{Evaluation protocol.}
We periodically evaluated saved LoRA checkpoints using greedy generation on the held-out evaluation split. GSM8K answers were normalized by extracting the final numeric answer, while BoolQ outputs were normalized to yes/no form. We report exact-match accuracy after normalization.

\subsubsection{Results and Analysis}
\label{subsubsec:emulator_results}
We analyzed the emulator results along three dimensions: overall robustness and utility, the effect of sparse LoRA aggregation, and the contribution of individual TERRA components.

\begin{table*}[t]
\centering
\caption{TERRA ablation at round 25 using Qwen2.5-0.5B; exact-match accuracy (\%) is reported across tasks and attack settings.}
\label{tab:terra_ablation_round25}
\small
\setlength{\tabcolsep}{10pt}
\begin{tabular}{lcccccc}
\toprule
& \multicolumn{3}{c}{GSM8K} & \multicolumn{3}{c}{BoolQ} \\
\cmidrule(lr){2-4} \cmidrule(lr){5-7}
Method & Clean & Signflip & Mixed & Clean & Signflip & Mixed \\
\midrule
Sparse FedAvg                      & 32.90 & 26.69 & 29.57 & 69.39 & 63.46 & 62.63 \\
Mask-aware aggregation             & \textbf{34.57} & 25.40 & 2.35  & 66.79 & 62.72 & 61.59 \\
Mask + norm filtering              & 32.90 & 30.63 & 28.05 & 68.59 & 68.69 & 68.75 \\
Mask + norm + cosine validation    & 27.67 & 25.40 & 28.96 & 68.23 & 70.52 & 70.98 \\
Full TERRA                        & 33.74 & \textbf{31.09} & \textbf{31.67} & \textbf{72.32} & \textbf{71.16} & \textbf{72.11} \\
\bottomrule
\end{tabular}
\end{table*}

\noindent\textbf{Overall robustness and utility.}
Table~\ref{tab:round25_emulator_qwen05b} highlights the final exact-match accuracy at global round 25 for Qwen2.5-0.5B under clean, sign-flip, and mixed attack settings. Overall, Thermo-FL achieved the strongest performance on BoolQ across all three settings, reaching 72.32\% in the clean setting, 71.16\% under sign-flip attacks, and 72.11\% under mixed attacks. This result suggests that TERRA was able to preserve useful sparse learning signal while limiting the effect of corrupted updates.

On GSM8K, Thermo-FL remained competitive but was not uniformly the best method in every setting. It achieved 33.74\% in the clean setting, 31.09\% under sign-flip attacks, and 31.67\% under mixed attacks. Several robust baselines performed strongly in individual GSM8K settings, especially Multi-Krum and coordinate-wise median. However, Thermo-FL consistently improved over sparse FedAvg under adversarial settings, increasing accuracy from 26.69\% to 31.09\% under sign-flip attacks and from 29.57\% to 31.67\% under mixed attacks. These results indicate that TERRA is most useful as a stability mechanism for sparse adversarial aggregation, rather than as a universal accuracy maximizer across all tasks.

\rqtakeaway{Finding 1: Robust sparse aggregation}
{TERRA improved robustness under adversarial sparse LoRA aggregation while preserving competitive utility. The gains were strongest on BoolQ, where Thermo-FL achieved the highest accuracy across clean, sign-flip, and mixed settings. On GSM8K, Thermo-FL remained competitive and improved over sparse FedAvg under all settings.}

\noindent\textbf{Dense versus sparse aggregation.}
The comparison between dense and sparse baselines shows that sparsification did not necessarily degrade utility. Sparse FedAvg achieved accuracy close to dense FedAvg in the clean setting and, in several adversarial cases, performed better than its dense counterpart. This suggests that top-\(k\) LoRA sparsification can reduce communication while preserving enough update signal for downstream adaptation. However, sparse updates also introduce support heterogeneity, which makes naive aggregation more fragile under attack. Thermo-FL addresses this issue by combining sparse transmission with mask-aware filtering and aggregation.

\rqtakeaway{Finding 2: Sparse transmission did not collapse utility}
{Top-\(k\) LoRA sparsification preserved enough update signal for downstream adaptation, but sparse support heterogeneity made naive aggregation more fragile under attack. This motivates mask-aware robust aggregation rather than sparse FedAvg alone.}

\noindent\textbf{Effect of TERRA components.}
Table~\ref{tab:terra_ablation_round25} reports the ablation study for TERRA. The results show that mask-aware aggregation alone was not sufficient as a robustness mechanism. In particular, under the mixed attack on GSM8K, mask-aware aggregation collapsed to 2.35\%, showing that preserving sparse support does not protect the server when adversarial clients manipulate update values or supports. Adding norm filtering substantially improved robustness by removing high-magnitude poisoned updates. Directional validation further constrained updates that had plausible magnitude but harmful orientation, while adaptive clipping bounded the coordinate-level influence of updates that passed the first two filters.
The full TERRA pipeline achieved the best overall balance across tasks and attack settings. On BoolQ, full TERRA obtained the highest accuracy in the clean, sign-flip, and mixed settings. On GSM8K, it improved over sparse FedAvg in both adversarial settings and avoided the severe collapse observed with mask-aware aggregation alone. These results support the design choice of combining norm filtering, mask-aware directional validation, adaptive active-coordinate clipping, and mask-aware aggregation as complementary stages rather than independent alternatives.
\rqtakeaway{Finding 3: TERRA stages are complementary}
{Mask-aware aggregation alone was not sufficient for robustness. Norm filtering, directional validation, adaptive clipping, and mask-aware aggregation worked together to limit high-magnitude, misaligned, and coordinate-concentrated adversarial updates.}

\subsection{Physical Testbed}
\label{subsec:physical_testbed}

The physical testbed complemented the emulator by validating whether Thermo-FL could run end-to-end on real edge hardware. Unlike the emulator, this setting captured measured device temperature, runtime variability, compressed payload construction, and attack behavior under Jetson-class client constraints. Because only two physical clients were available, this study was used as a prototype validation of hardware behavior rather than as a large-scale robustness benchmark.

\subsubsection{Experimental Setup}
\label{subsubsec:physical_setup}

The physical prototype was configured to evaluate feasibility, thermal behavior, communication overhead, and representative attack response under realistic edge-device constraints.

\noindent\textbf{Model and dataset.}
We used Qwen2.5-0.5B as the base model and fine-tuned it with LoRA under 8-bit quantization. The pretrained backbone was frozen, and only the LoRA adapter parameters were updated. Training and evaluation were performed on GSM8K, using 7,473 training examples and 1,319 held-out test examples.

\noindent\textbf{Hardware testbed.}
The testbed consisted of two NVIDIA Jetson Orin Nano Developer Kits as edge clients and an Acer Aspire E15 laptop as the aggregation server. Each Jetson client had 8\,GB of RAM and a 1\,TB NVMe SSD. The server used an 8th-generation Intel Core i5 processor with 12\,GB of RAM.

\noindent\textbf{Data partitioning and training.}
The GSM8K training set was divided into two disjoint client subsets of 3,736 and 3,737 samples. All physical experiments were run for 200 communication rounds. In each round, both clients performed 5 local optimization steps using AdamW with learning rate \(3\times10^{-4}\), batch size 1, and gradient accumulation of 4. LoRA was applied to the query, key, value, and output projection matrices with rank \(r=8\) and scaling parameter \(\alpha=16\).

\noindent\textbf{Thermal control and telemetry.}
Device telemetry was sampled every \(1\,\mathrm{Hz}\) using \texttt{tegrastats}. Experiments were conducted in a climate-controlled laboratory at approximately \(25^\circ\mathrm{C}\). Devices were placed on a non-conducting flat surface without external cooling and relied only on their built-in cooling systems. The controller used \(T_{\min}=45^\circ\mathrm{C}\) and \(T_{\max}=65^\circ\mathrm{C}\). The sparse transmission ratio \(\rho\) was adjusted between 50\% and 1\% using a clipped linear policy. The trainable-layer ratio \(\kappa\) followed a three-level policy: \(\kappa=1.0\) for \(T<45^\circ\mathrm{C}\), \(\kappa=0.5\) for \(45^\circ\mathrm{C}\leq T\leq65^\circ\mathrm{C}\), and \(\kappa=0.1\) for \(T>65^\circ\mathrm{C}\).

\noindent\textbf{Communication encoding.}
We evaluated dense-full, dense-sparse, COO, bitmap, flat-index, flat-delta, and values-only encodings(ideal metrics). All schemes followed the same pipeline: payload construction, binary serialization, and zlib lossless compression before upload. Compressed upload size was used as the primary communication metric, while encoding and compression time were measured as secondary overhead.

\noindent\textbf{Baselines and scope.}
The physical testbed was used to study end-to-end hardware behavior rather than to reproduce the full emulator baseline comparison. We evaluated three configurations: \textit{FedAvg-full}, which averaged full-density LoRA adapter states with standard FedAvg; \textit{Thermo-FL without TERRA}, which used thermal-aware client control and sparse transmission with FedAvg-style aggregation; and \textit{Thermo-FL}, which used thermal-aware control, sparse encoding, and TERRA aggregation. In all cases, the pretrained backbone remained frozen, and only the LoRA adapter state was communicated or aggregated. Since clients started each round from the same global adapter state, averaging local adapter states was equivalent to averaging the corresponding adapter update differences.

\noindent\textbf{Adversarial settings.}
We evaluated clean execution, one internal Byzantine attack represented by sign-flip/scale, and one communication-layer perturbation. In the Byzantine setting, one client submitted sign-scaled updates with \(\gamma=-10\). In the MITM setting, Gaussian noise proportional to the update standard deviation was injected into transmitted updates. We did not use the mixed-attack setting on the physical testbed because only two clients participated; with such a small client population, randomly mixing several attack operators would reduce to a single-client perturbation rather than a meaningful distributional robustness test. Mixed attacks were therefore evaluated in the large-scale emulator, while the physical testbed isolated representative endpoint-level and communication-layer corruption cases.
\subsubsection{Results and Analysis}
\label{subsubsec:physical_results}
We evaluated the physical prototype along four axes, which are thermal behavior, system efficiency, communication overhead, and robustness under representative attacks.

\begin{figure}[t]
  \centering
  \includegraphics[width=.8\linewidth]{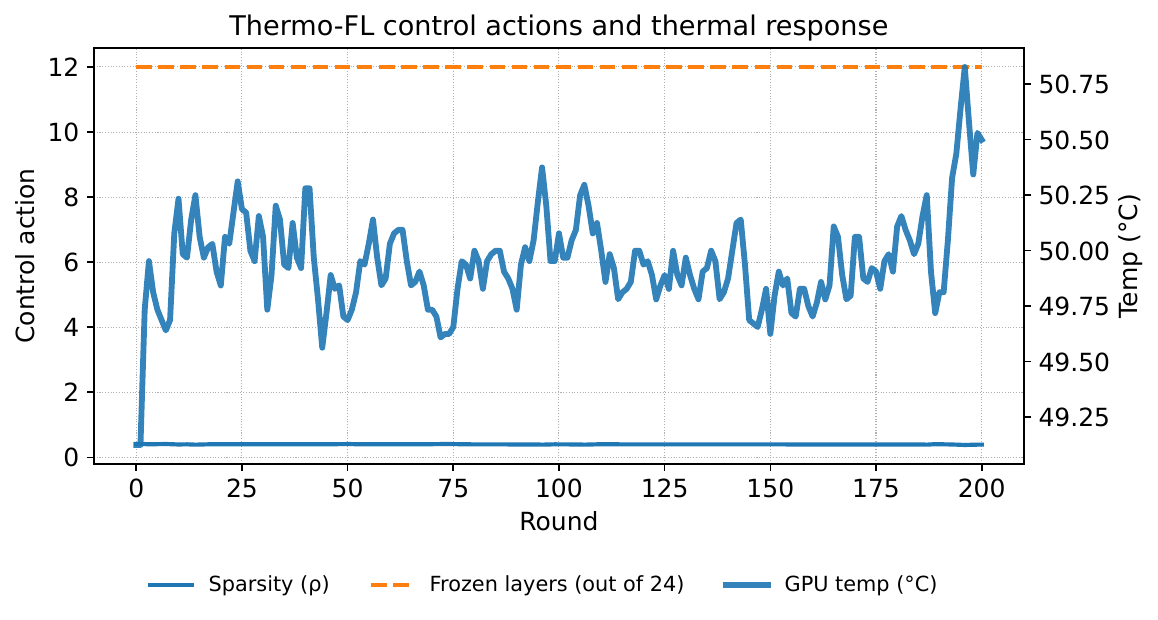}
  \caption{Thermal controller behavior across 200 communication rounds. Thermo-FL adjusted sparse transmission and trainable-layer control in response to measured GPU temperature.}
  \label{fig:rq1_control}
\end{figure}

\begin{figure}[t]
    \centering
    \begin{subfigure}[b]{0.48\linewidth}
        \centering
        \includegraphics[width=\linewidth]{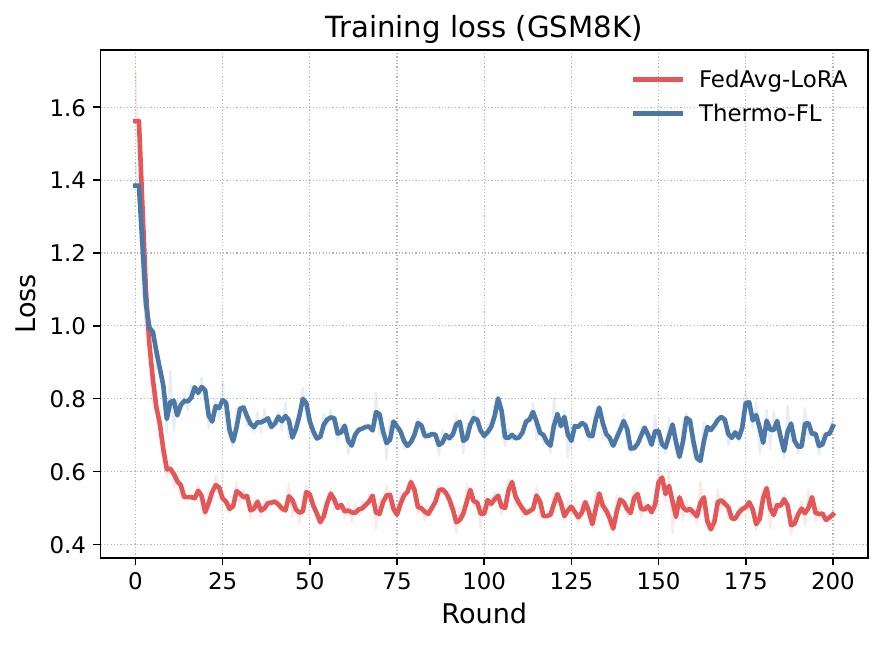}
        \caption{Training loss}
    \end{subfigure}
    \hfill
    \begin{subfigure}[b]{0.48\linewidth}
        \centering
        \includegraphics[width=\linewidth]{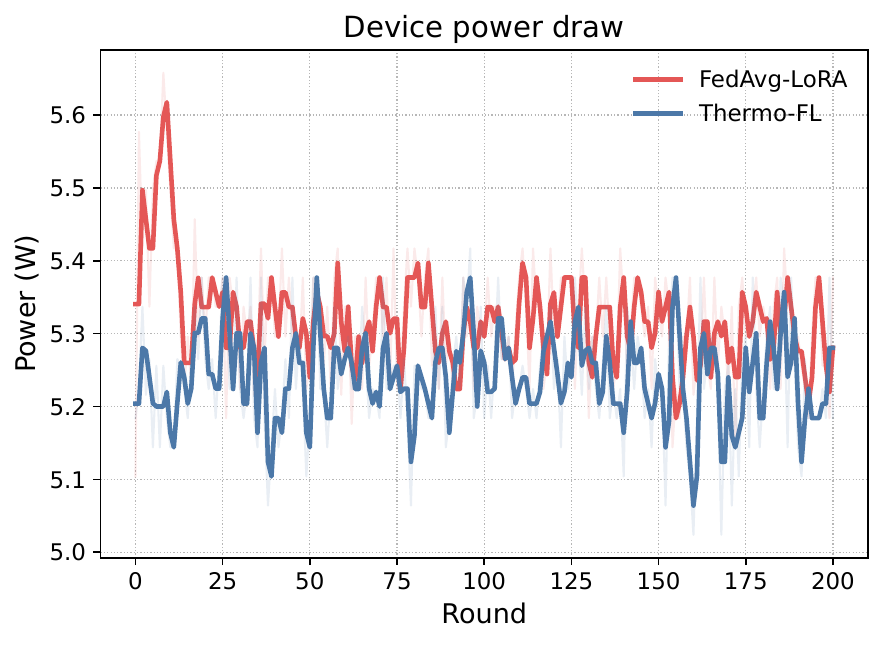}
        \caption{Power}
    \end{subfigure}

    \vspace{0.3em}

    \begin{subfigure}[b]{0.48\linewidth}
        \centering
        \includegraphics[width=\linewidth]{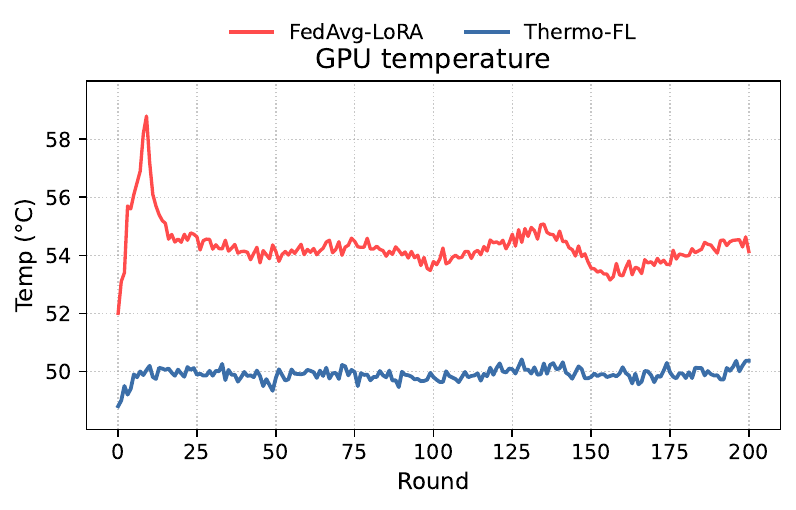}
        \caption{GPU temperature}
    \end{subfigure}
    \hfill
    \begin{subfigure}[b]{0.48\linewidth}
        \centering
        \includegraphics[width=\linewidth]{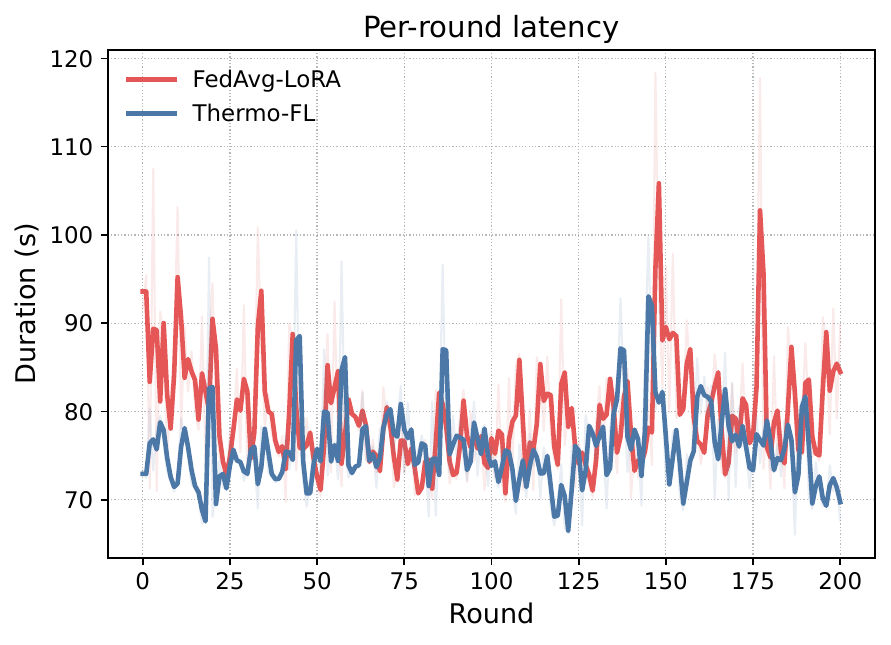}
        \caption{Round latency}
    \end{subfigure}

    \caption{Clean physical-testbed behavior for Thermo-FL and FedAvg-LoRA. The comparison includes training loss, power consumption, GPU temperature, and per-round latency.}
    \label{fig:physical_clean_behavior}
\end{figure}

\begin{figure}[t]
  \centering
  \begin{subfigure}{0.45\linewidth}
    \centering
    \includegraphics[width=\linewidth]{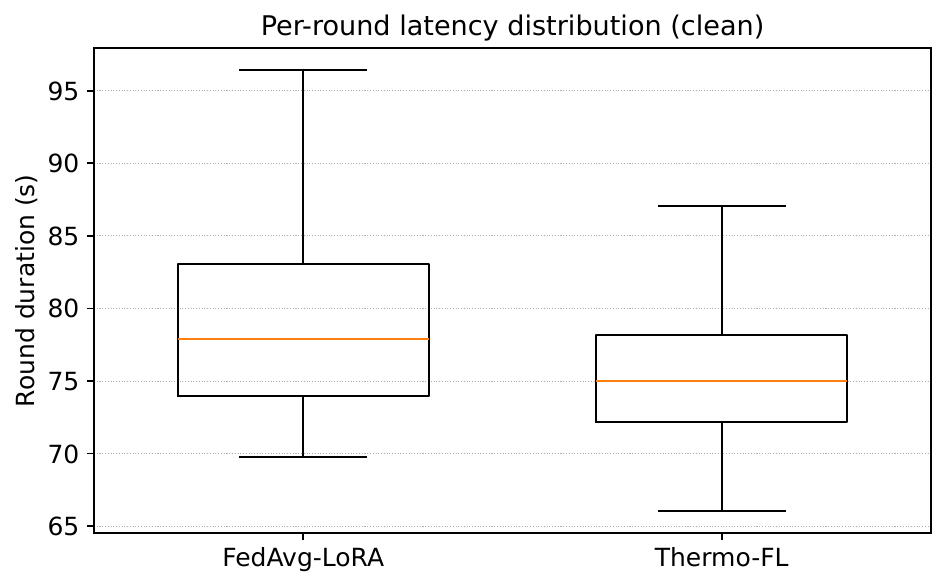}
    \caption{Latency}
  \end{subfigure}
  \hfill
  \begin{subfigure}{0.45\linewidth}
    \centering
    \includegraphics[width=\linewidth]{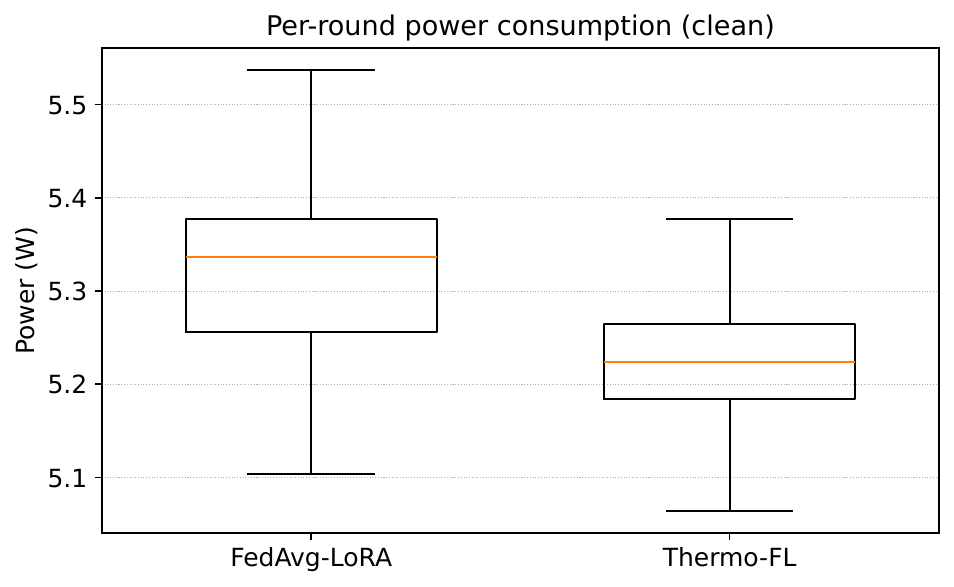}
    \caption{Power}
  \end{subfigure}
  \caption{Distribution of per-round latency and power consumption under clean physical execution.}
  \label{fig:physical_latency_power_dist}
\end{figure}

\noindent\textbf{Thermal behavior and system efficiency.}
Figure~\ref{fig:rq1_control} shows that Thermo-FL actively used measured temperature to regulate the physical training process. Across 200 communication rounds, the controller kept GPU temperature within a narrow range of roughly \(49\text{-}51^\circ\mathrm{C}\). As temperature increased, Thermo-FL reduced local training workload and transmitted sparser updates; when the device cooled, these constraints were relaxed. This confirms that temperature was not only logged but used as a feedback signal for federated fine-tuning.

Figure~\ref{fig:physical_clean_behavior} compares Thermo-FL with FedAvg-LoRA under clean execution. FedAvg-LoRA showed an early temperature spike near \(60^\circ\mathrm{C}\) and then operated around \(54\text{-}55^\circ\mathrm{C}\) for much of training. In contrast, Thermo-FL stabilized near \(50^\circ\mathrm{C}\), while also reducing power consumption and round-latency variation. Figure~\ref{fig:physical_latency_power_dist} further shows that Thermo-FL had lower median latency and a tighter latency distribution, indicating fewer thermally induced straggler effects and more predictable round completion.

\begin{figure}[t]
    \centering
    \begin{subfigure}[t]{\linewidth}
        \centering
        \includegraphics[width=.80\linewidth]{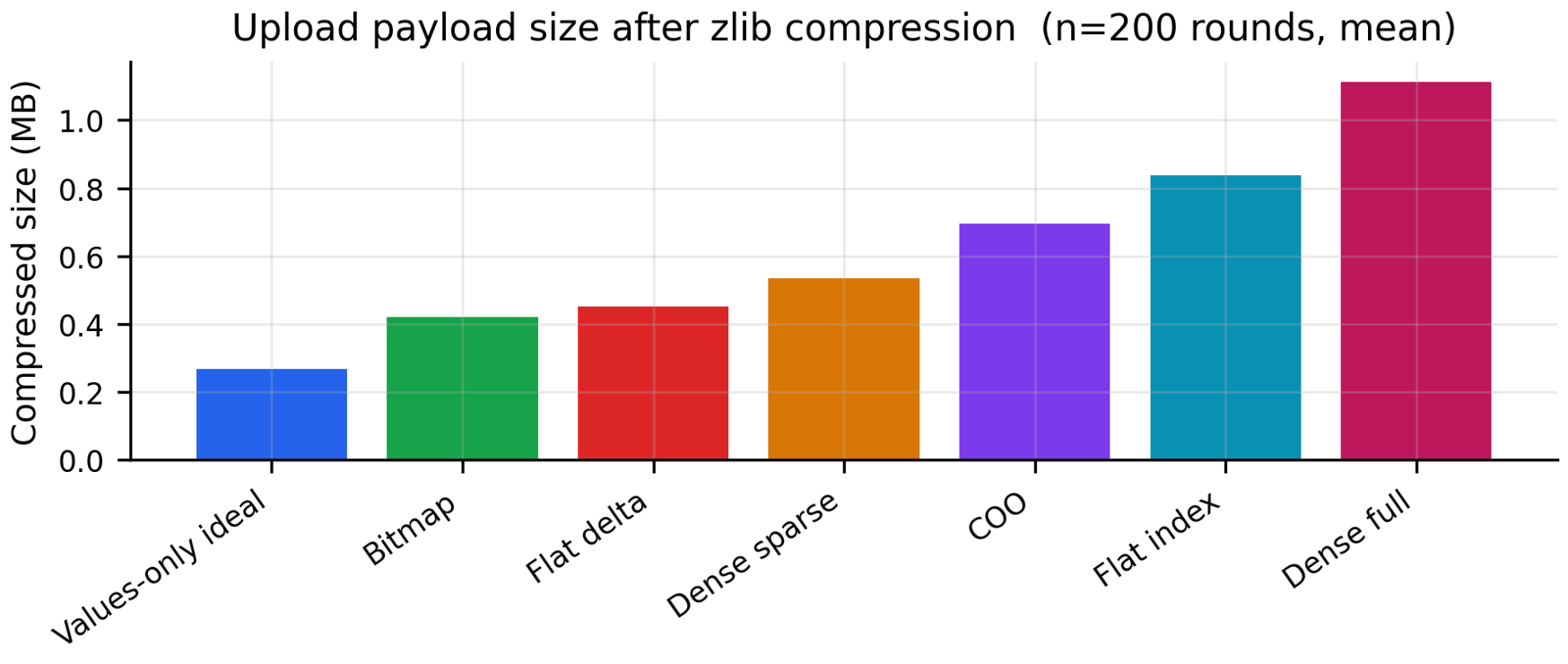}
        \caption{Mean compressed upload size after serialization and compression.}
        \label{fig:physical_payload_size}
    \end{subfigure}

    \vspace{0.4em}

    \begin{subfigure}[t]{\linewidth}
        \centering
        \includegraphics[width=.80\linewidth]{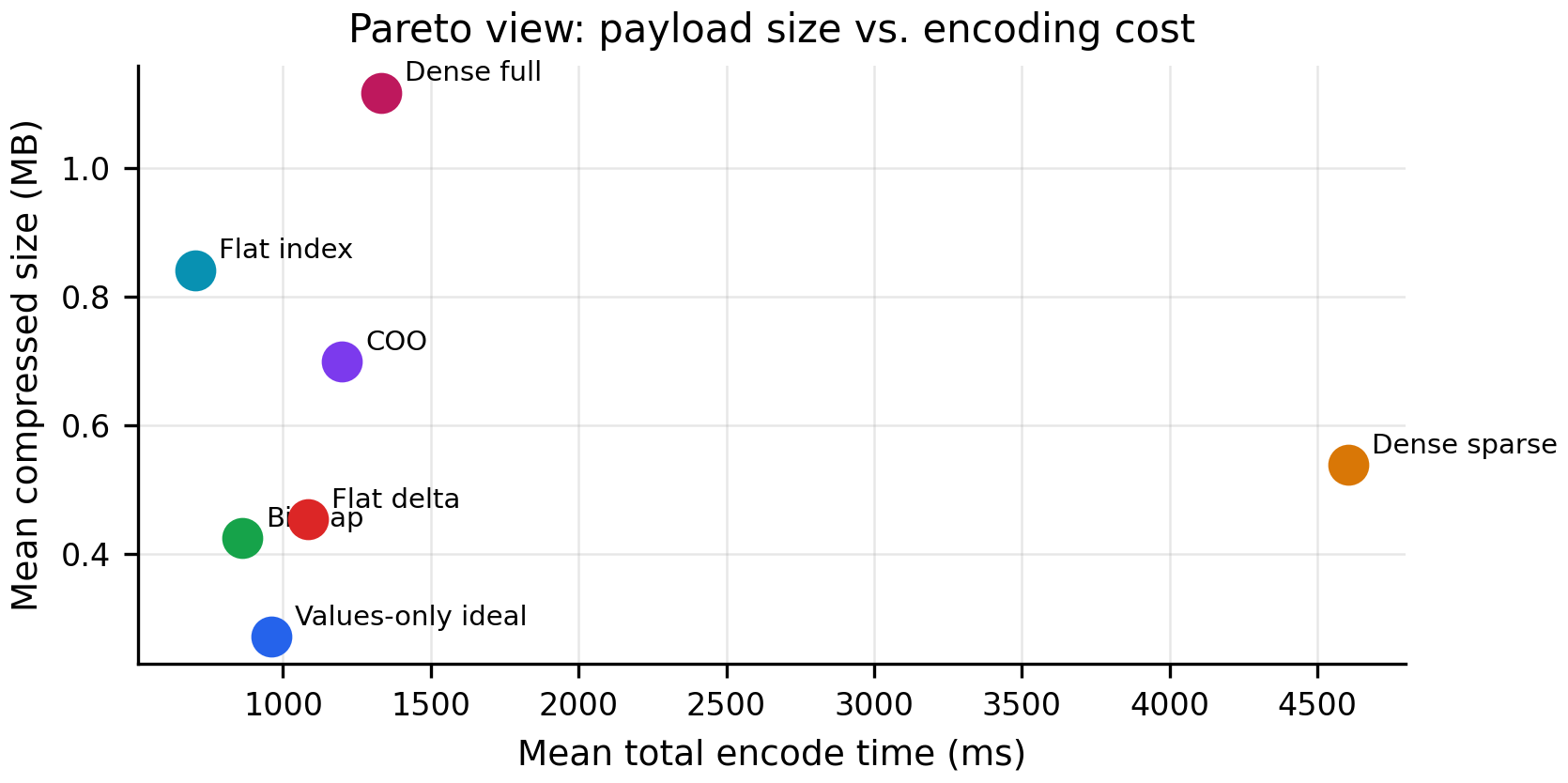}
        \caption{Compressed payload size versus encoding cost.}
        \label{fig:physical_payload_pareto}
    \end{subfigure}

    \caption{Communication-efficiency benchmark for sparse update encodings on the physical testbed.}
    \label{fig:physical_comm}
\end{figure}

\noindent\textbf{Communication overhead.}
Figure~\ref{fig:physical_comm} reports the communication benchmark for candidate sparse-update encodings. All schemes used the same pipeline: payload construction, binary serialization, and zlib lossless compression before upload. The raw dense LoRA payload was approximately \(2.23\,\mathrm{MB}\), while compressed dense-full transmission required about \(1.11\,\mathrm{MB}\). Among deployable sparse formats, bitmap achieved the smallest practical compressed payload at approximately \(0.42\,\mathrm{MB}\). Flat-delta was close at \(0.45\,\mathrm{MB}\), followed by dense-sparse at \(0.54\,\mathrm{MB}\), COO at \(0.65\,\mathrm{MB}\), and flat-index at \(0.84\,\mathrm{MB}\). The values-only format reached \(0.27\,\mathrm{MB}\), but it served only as an ideal lower bound because it assumes zero cost for support metadata. Overall, bitmap provided the best practical size-overhead tradeoff.

\begin{figure}[t]
  \centering
  \begin{subfigure}[t]{0.48\linewidth}
    \centering
    \includegraphics[width=\linewidth]{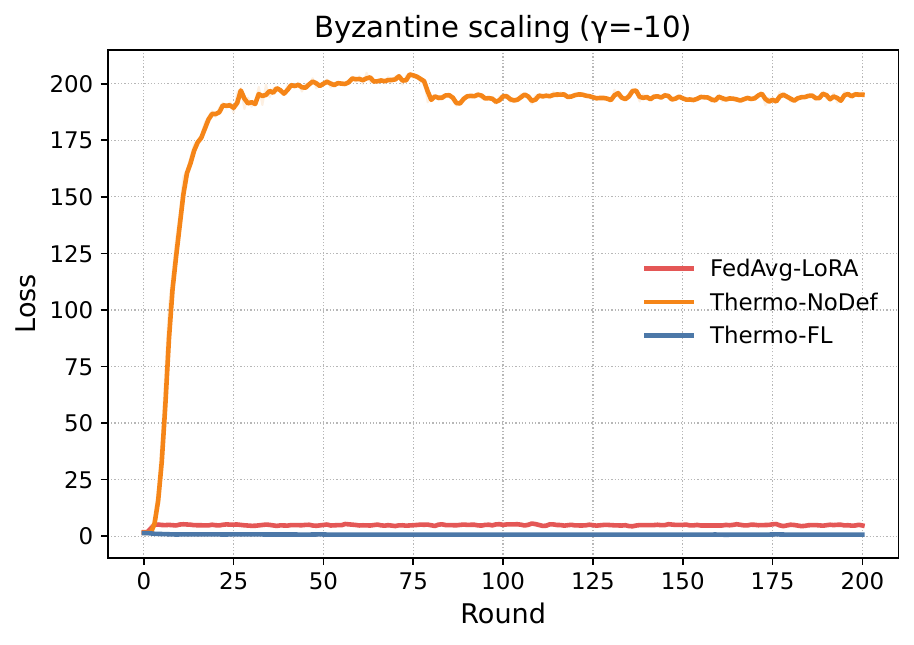}
    \caption{Sign-flip/scale attack \((\gamma=-10)\)}
    \label{fig:physical_byz_loss}
  \end{subfigure}
  \hfill
  \begin{subfigure}[t]{0.48\linewidth}
    \centering
    \includegraphics[width=\linewidth]{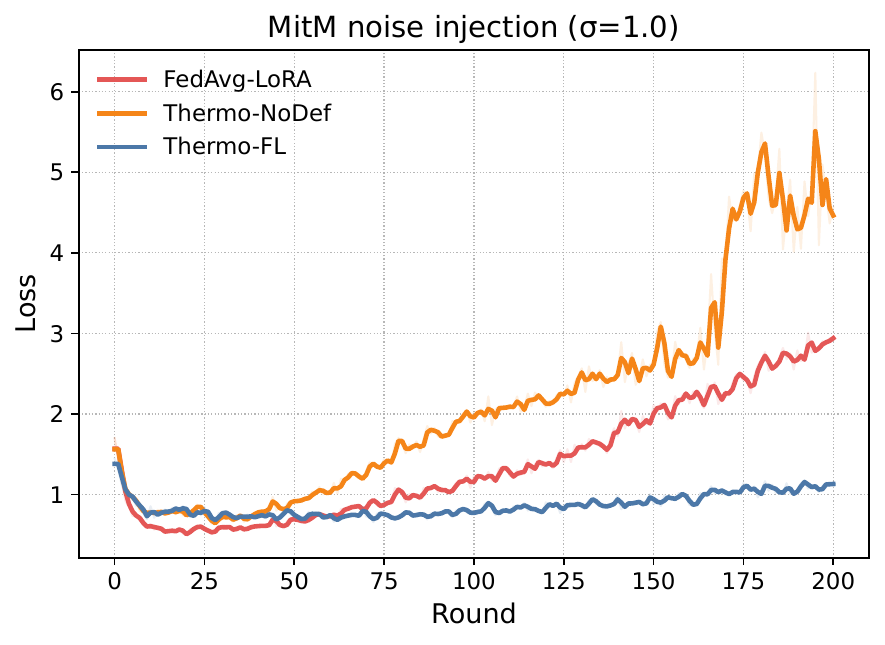}
    \caption{MITM noise injection \((\sigma=1.0)\)}
    \label{fig:physical_mitm_loss}
  \end{subfigure}

  \caption{Training loss under representative physical attack settings.}
  \label{fig:physical_attacks}
\end{figure}

\noindent\textbf{Robustness under physical attacks.}
Figure~\ref{fig:physical_attacks} shows the training loss under Byzantine scaling and MITM noise injection. Under the Byzantine scaling attack with \(\gamma=-10\), Thermo-FL without TERRA became unstable, whereas the full Thermo-FL pipeline remained stable. A similar pattern appeared under MITM noise injection: FedAvg-LoRA degraded after convergence, Thermo-FL without TERRA became unstable, and Thermo-FL maintained a stable trajectory. These results indicate that TERRA's norm filtering, directional validation, and adaptive clipping constrained corrupted updates even in the small-client physical setting.

\begin{table}[t]
\centering
\caption{Physical-testbed GSM8K exact-match accuracy (\%) under clean and adversarial settings. Higher is better.}
\label{tab:gsm8k_physical_utility}
\small
\begin{tabular}{lccc}
\toprule
\textbf{Method} & \textbf{Clean} & \textbf{Sign-flip/scale} & \textbf{MITM} \\
\midrule
FedAvg-LoRA & 23.28 & 0.00 & 0.00 \\
Thermo-FL without TERRA & -- & 0.00 & 0.00 \\
\textbf{Thermo-FL} & \textbf{24.64} & \textbf{21.40} & \textbf{18.65} \\
\bottomrule
\end{tabular}
\vspace{0.3ex}
\begin{minipage}{\columnwidth}
\footnotesize
\textit{Note:} The sign-flip/scale attack used \(\gamma=-10\). We used 8-Bit Quantized Qwen2.5-0.5B model.
\end{minipage}
\end{table}

\noindent\textbf{Downstream task utility.}
Table~\ref{tab:gsm8k_physical_utility} reports GSM8K exact-match accuracy on the physical testbed. In the clean setting, Thermo-FL achieved \(24.64\%\), slightly higher than FedAvg-LoRA at \(23.28\%\). Under attack, FedAvg-LoRA collapsed to \(0.00\%\), and Thermo-FL without TERRA also failed under both sign-flip/scale and MITM noise injection. In contrast, the full Thermo-FL pipeline preserved \(21.40\%\) accuracy under sign-flip/scale and \(18.65\%\) under MITM perturbation. These results show that the physical gains were not limited to smoother hardware behavior; TERRA was necessary to preserve downstream task utility when corrupted updates reached the server.

\rqtakeaway{Finding 4: Physical feasibility and robustness}
{Thermo-FL ran end-to-end on Jetson-class edge clients, stabilized device temperature, reduced compressed upload size through bitmap sparse encoding, and preserved GSM8K utility under representative sign-flip/scale and MITM attacks.}
\section{Discussion and Limitations}
\label{sec:discussion}
While Thermo-FL proves that thermal-aware client control and strong sparse aggregation techniques can be combined to enable edge LLM fine-tuning, there are still several limitations. First of all, the physical testbed used was intentionally small, consisting of only two Jetson Orin Nano as clients and one laptop server. Such architecture allowed validating end-to-end operation, telemetry measurement, payload compression, and attack simulation on actual hardware, but it was not intended to replace a large-scale emulator. Consequently, the emulator setup helped perform comparisons across larger client population, non-IID data partitions, and different adversary types.

Secondly, the physical attack assessment included only one sign-flip/scale attack and one attack based on MITM-style noise injection. These attacks were used to isolate representative endpoint-level and communication-layer corruption in a small-client deployment. The mixed attacks were only simulated in the emulator, where the number of clients was sufficient enough to make stochastic attack sampling meaningful. Third, TERRA pipeline of Thermo-FL framework does not eliminate the need for any cryptographic transport protection, client authentication, secure boot, or trusted sensing mechanisms. Rather, TERRA provides an aggregation-layer defense that limits the influence of corrupted sparse updates when such updates can penetrate the existing defences and reaches the server.

Last but not least, the temperature control policy of Thermo-FL used fixed temperature threshold and hardware-specific parameters for \(\kappa\) and \(\rho\). These values were appropriate for the Jetson-class prototype, but other edge devices may require calibration of those parameters and threshold based on their cooling capacity, power mode, workload intensity, and deployment environment. Future work can extend Thermo-FL with adaptive threshold selection, broader device heterogeneity, additional LLM tasks, and larger physical federated deployments.
\section{Related Work}
\label{sec:related_work}

\textbf{Federated learning (FL)} allows clients which are distributed across various location to collaboratively fine-tune or adapt shared models without centralizing raw data~\cite{mcmahan2023communicationefficientlearningdeepnetworks, khouas2024trainingmachinelearningmodels, openprobleminFL}. This makes FL an attractive solution for edge and on-device LLM settings, where local data may be private, domain-specific, or costly to transmit~\cite{EdgeLLMsurvey, xu2024ondevicelanguagemodelscomprehensive, xu2025surveyprivacysecuritymobile, ZhuSurvey}. Since full fine-tuning of LLMs poses a significant challenges due to their immense size and resource requirements, parameter-efficient fine-tuning (PEFT) methods, like  LoRA~\cite{hu2021loralowrankadaptationlarge}, prompt tuning~\cite{lester2021prompttuning}, and prefix tuning~\cite{li2021prefixtuningoptimizingcontinuousprompts}, have become central to efficient adaptation, which reduce the number of trainable parameters while preserving the adaptation capability. Recent federated LLM fine-tuning methods build on these techniques to support the collaborative adaptation across heterogeneous clients, and several works further reduces the communication through sparse or compressed LoRA updates~\cite{kuo2024federatedlorasparsecommunication, liu2025ecoloracommunicationefficientfederatedfinetuning, yan2026fedsrdsparsifyreconstructdecomposecommunicationefficientfederated}. These research improves the practicality of the federated LLM adaptation, but they generally treat efficiency as an algorithmic objective rather than as a response to the changing hardware states.

\textbf{Edge deployment of LLM introduces hardware dynamics and limitation that are difficult to capture with the conventional FL abstractions.} Edge client devices usually operate under the limited power, constrained cooling, battery dependence, and dynamic voltage frequency scaling (DVFS), which can affect the runtime, throughput, and overall efficiency of the devices~\cite{HungDVFS, GadeDVFS, HanimaiahDVFS, HanimaiahDVFSPerformance, LeDVFS, LuiDVFS}, which in-turn increases the time required to fine-tune the models in the federated learning setup. Thus in synchronous FL, such runtime variation can delay the aggregation and reduce the overall system efficiency, especially when some clients become stragglers or participate inconsistently across rounds~\cite{li2020federatedoptimizationheterogeneousnetworks, haddadpour2019convergencelocaldescentmethods, ZhuSurvey}. Prior thermal-aware and DVFS-aware systems regulate workloads, voltage-frequency states, or operating system-scheduling policies to improve the energy efficiency and avoid throttling~\cite{HungDVFS, GadeDVFS, LuiDVFS}. These approaches are often very essential in the edge devices and federated fine-tuning systems, but they typically treat temperature as a runtime-management variable instead of a signal that shapes the federated fine-tuning behaviour.

\textbf{Federated learning is also vulnerable to the unreliable and malicious update}, caused by compromised clients or communication channels. Byzantine clients, model poisoning, backdoor attacks, and compromised communication channels can corrupt the updates which are used to form the global model~\cite{bagdasaryanBackdoorFL, BhagojiFLAdverserial, jimenezgutierrez2025securityprivacyfederatedlearning, ContiMITM}. Existing Robust aggregation methods reduces this risk of corrupted global model by rejecting abnormal updates or limiting the influence of co-ordinates which are outliers. Krum and Multi-Krum selects the updates using the distance based consistency~\cite{KRUM, multiKrum}. Bulyan combines the selection with co-ordinate wise consistency aggregation~\cite{bulyan}, whereas co-ordinate wise median and trimmed mean methods reduces the effect of co-ordinate level outliers~\cite{Median_Mean}. These defenses provides important method to protect the global model from being corrupted in the adversarial FL setup, but they are usually designed for the dense or structurally comparable updates.  However, this paper studies a setting where client-side hardware constraints can change participation and update structure, requiring thermal-aware client adaptation and robust aggregation to be considered together in a single framework.

To sum it up, prior works have made substantial progress in the field of federated LLM adaptation, thermal-aware edge execution, and robust aggregation. However, these fields are typically studied separately. This leaves an open problem for edge LLM fine-tuning, how to adapt client-side training to hardware constraints while maintaining server-side robustness against unreliable or adversarial updates.

\section{Conclusion}
\label{sec:conclusion}
In this research work, we proposed Thermo-FL, a thermal-aware federated fine-tuning framework for large language models running on edge clients. In the proposed framework, device temperature acts as a control signal regulating local LoRA training and sparse update transmission, thus lowering the client workload under thermal stress. In order to protect aggregation mechanism under dynamically sparse and adversarial updates, TERRA was introduced as a robust aggregation pipeline that combines norm filtering, mask-aware directional validation, adaptive active-coordinate clipping, and mask-aware aggregation.

Thermo-FL was evaluated using both a large-scale emulator and a physical Jetson-based testbed. The emulator showed improved robustness under adversarial sparse aggregation while preserving competitive utility, especially on BoolQ. The physical prototype further demonstrated end-to-end feasibility on real edge hardware, with stabilized device temperature, reduced compressed upload size through bitmap sparse encoding, and preserved GSM8K utility under sign-flip/scale and MITM perturbations. Overall, these results suggest that secure edge LLM adaptation should treat hardware behavior, sparse communication, and aggregation robustness as coupled design requirements rather than separate system concerns.

\newpage
\bibliographystyle{IEEEtran}
\bibliography{references}

\newpage
\appendix
\section{Convergence Rationale Under an Abstract TERRA Operator}
\label{app:convergence}
This appendix provides a compact convergence rationale for Thermo-FL under thermally regulated local computation, temperature-driven sparsification, and server-side TERRA aggregation. The implemented TERRA pipeline is data-dependent: the accepted client set, clipping threshold, and coordinate supports vary by round. A full first-principles convergence proof for the exact implementation would therefore require strong assumptions on the evolution of these data-dependent quantities. Instead, we analyze a faithful abstraction in which the implemented TERRA aggregate satisfies explicit alignment and bounded-moment conditions. This is the standard role of the appendix: to clarify the optimization behavior implied by the design, not to claim unconditional convergence under arbitrary Byzantine behavior.

\subsection{Notation Bridge}

The main text writes the server update as
\begin{equation}
W^{(t+1)} = W^{(t)} + \bar{\Delta}^{(t)} ,
\label{eq:app_server_update}
\end{equation}
where \(\bar{\Delta}^{(t)}\) is the aggregate LoRA delta produced by TERRA. For the convergence argument, we define the associated descent direction
\begin{equation}
G^{(t)}
:=
-\frac{1}{\eta}\bar{\Delta}^{(t)} ,
\label{eq:app_descent_direction}
\end{equation}
so that Eq.~\eqref{eq:app_server_update} can be written equivalently as
\begin{equation}
W^{(t+1)}
=
W^{(t)}-\eta G^{(t)} .
\label{eq:app_descent_update}
\end{equation}

We consider the weighted federated objective
\begin{equation}
J(W)=\sum_{k=1}^{K}p_kF_k(W),
\qquad
\sum_{k=1}^{K}p_k=1,
\label{eq:app_objective}
\end{equation}
where \(F_k\) is the local empirical risk of client \(k\).

At round \(t\), a selected client \(k\in S_t\) computes a stochastic gradient \(g_k^{(t)}\), applies the layer-freezing projection induced by \(\kappa_k^{(t)}\), and then applies temperature-driven top-\(k\) sparsification:
\begin{equation}
v_k^{(t)}
=
S_{\rho_k^{(t)}}
\!\left(
P_{\kappa_k^{(t)}}g_k^{(t)}
\right).
\label{eq:app_projected_sparse_direction}
\end{equation}
The corresponding sparse LoRA delta is
\begin{equation}
\tilde{\Delta}_k^{(t)}
=
-\eta v_k^{(t)} .
\label{eq:app_sparse_delta}
\end{equation}
Let \(m_k^{(t)}\in\{0,1\}^{d}\) denote the decoded support mask. Consistent with the main text, the decoded sparse update object is
\begin{equation}
u_k^{(t)}
=
\left(
\tilde{\Delta}_k^{(t)},m_k^{(t)}
\right),
\qquad
U^{(t)}
=
\{u_k^{(t)}:k\in S_t\}.
\label{eq:app_decoded_update_set}
\end{equation}
The implemented TERRA aggregate is written as
\begin{equation}
\bar{\Delta}^{(t)}
=
\mathcal{F}_{\mathrm{TERRA}}^{(t)}(U^{(t)}),
\label{eq:app_terra_operator}
\end{equation}
where \(\mathcal{F}_{\mathrm{TERRA}}^{(t)}\) includes norm filtering, mask-aware directional validation, adaptive active-coordinate clipping, and mask-aware aggregation.

\subsection{Assumptions}

\paragraph{Smoothness.}
The global objective \(J\) is \(L\)-smooth:
\begin{equation}
\|\nabla J(x)-\nabla J(y)\|_2
\leq
L\|x-y\|_2 ,
\qquad
\forall x,y .
\label{eq:app_smoothness}
\end{equation}

\paragraph{Stochastic gradients.}
For each client \(k\), the stochastic gradient is unbiased with bounded variance:
\begin{equation}
\mathbb{E}[g_k^{(t)}(W)]
=
\nabla F_k(W),
\qquad
\mathbb{E}\!\left[
\|g_k^{(t)}(W)-\nabla F_k(W)\|_2^2
\right]
\leq
\sigma^2 .
\label{eq:app_gradient_assumption}
\end{equation}

\paragraph{Thermal control bounds.}
The thermal controller keeps the active-layer and transmission ratios within fixed positive ranges:
\begin{equation}
\kappa_k^{(t)}\geq \kappa_{\min}>0,
\qquad
\rho_{\min}\leq \rho_k^{(t)}\leq \rho_{\max}\leq 1 .
\label{eq:app_thermal_bounds}
\end{equation}
The layer-freezing projection may introduce a bounded bias:
\begin{equation}
\left\|
\nabla J(W^{(t)})
-
P_{\kappa_k^{(t)}}\nabla J(W^{(t)})
\right\|_2^2
\leq
\epsilon_{\kappa_k^{(t)}},
\qquad
\sup_{t,k}\epsilon_{\kappa_k^{(t)}}\leq \bar{\epsilon}.
\label{eq:app_projection_bias}
\end{equation}

\paragraph{Sparsification.}
The top-\(k\) sparsifier preserves a nontrivial fraction of the update energy. For any vector \(x\),
\begin{equation}
\langle x,S_{\rho}(x)\rangle
\geq
\rho\|x\|_2^2,
\qquad
\|S_{\rho}(x)\|_2^2
\leq
\|x\|_2^2 .
\label{eq:app_sparsifier_property}
\end{equation}

\paragraph{Abstract TERRA alignment.}
There exist constants \(\alpha_{\mathrm{T}}\in(0,1]\) and \(\varepsilon_{\mathrm{T}}\geq0\) such that the descent direction produced by the implemented TERRA aggregate satisfies
\begin{equation}
\mathbb{E}
\left[
\left\langle
\nabla J(W^{(t)}),G^{(t)}
\right\rangle
\right]
\geq
\alpha_{\mathrm{T}}\rho_{\min}
\|\nabla J(W^{(t)})\|_2^2
-
(\bar{\epsilon}+\varepsilon_{\mathrm{T}}).
\label{eq:app_alignment}
\end{equation}
Here, \(\varepsilon_{\mathrm{T}}\) captures the residual bias introduced by robust filtering, clipping, and adversarial contamination that survives the filters.

\paragraph{Abstract TERRA second moment.}
There exist constants \(C_{\mathrm{T}}\geq1\) and \(\nu_{\mathrm{T}}\geq0\) such that
\begin{equation}
\mathbb{E}
\left[
\|G^{(t)}\|_2^2
\right]
\leq
C_{\mathrm{T}}
\left(
\|\nabla J(W^{(t)})\|_2^2+\sigma^2
\right)
+
\nu_{\mathrm{T}} .
\label{eq:app_second_moment}
\end{equation}
The term \(\nu_{\mathrm{T}}\) captures the residual second-moment contribution induced by filtering, clipping, and surviving adversarial components.

\subsection{Bounded Influence of the Implemented TERRA Stages}

The abstract constants above are not assumed to arise magically. The implemented TERRA stages provide bounded-influence behavior that supports these assumptions. In particular, norm filtering rejects updates whose sparse-delta norm exceeds the round-adaptive threshold \(B_t=m_s^{(t)}+\lambda_{\mathrm{norm}}\operatorname{MAD}^{(t)}\). Thus, a scaling or model-replacement update \(\tilde{\Delta}_a^{(t)}=\gamma\Delta_a^{(t)}\) is removed whenever
\begin{equation}
|\gamma|\,\|\Delta_a^{(t)}\|_2>B_t .
\label{eq:app_scaling_rejection}
\end{equation}

Mask-aware directional validation rejects updates whose alignment with the reference direction \(r^{(t)}\), measured only on the transmitted support, falls below \(\tau_{\mathrm{dir}}\). Hence sign-flipped updates that reverse the direction of an otherwise aligned sparse update are rejected when their support-restricted cosine score falls below the threshold.

For any update that survives filtering, adaptive clipping ensures
\begin{equation}
|\hat{\Delta}_{k,j}^{(t)}|
\leq
m_{k,j}^{(t)}q_t ,
\label{eq:app_coordinate_clip}
\end{equation}
and if at most \(\rho_{\max}d\) coordinates are transmitted,
\begin{equation}
\|\hat{\Delta}_k^{(t)}\|_2
\leq
q_t\sqrt{\rho_{\max}d}.
\label{eq:app_clipped_norm_bound}
\end{equation}
Finally, mask-aware aggregation bounds the coordinate-wise contribution of retained adversarial clients. Let
\begin{equation}
D_j^{(t)}
=
\sum_{k\in S_{\mathrm{dir}}^{(t)}}
w_km_{k,j}^{(t)}
+
\epsilon
\label{eq:app_active_mass}
\end{equation}
be the active aggregation mass for coordinate \(j\), and let
\[
A_j^{(t)}
=
\{a\in A_t\cap S_{\mathrm{dir}}^{(t)}:m_{a,j}^{(t)}=1\}
\]
be the retained adversarial clients that transmitted coordinate \(j\). Since clipping gives \(|\hat{\Delta}_{a,j}^{(t)}|\leq q_t\), their aggregate contribution satisfies
\begin{equation}
\left|
\frac{
\sum_{a\in A_j^{(t)}}
w_a m_{a,j}^{(t)}\hat{\Delta}_{a,j}^{(t)}
}{
D_j^{(t)}
}
\right|
\leq
\frac{
\sum_{a\in A_j^{(t)}} w_a
}{
D_j^{(t)}
}
q_t .
\label{eq:app_adversarial_contribution}
\end{equation}
This does not prove that every malicious update is removed. It shows that any retained adversarial influence is bounded by the clipping threshold and the adversarial active aggregation weight on each coordinate. This bounded-influence property is the implementation-level rationale behind the abstract residual terms \(\varepsilon_{\mathrm{T}}\) and \(\nu_{\mathrm{T}}\).

\subsection{Convergence Result}

\begin{theorem}[Stationary-point bound under abstract TERRA]
\label{thm:app_convergence}
Assume Eqs.~\eqref{eq:app_smoothness}--\eqref{eq:app_second_moment}. If
\begin{equation}
\eta
\leq
\frac{\alpha_{\mathrm{T}}\rho_{\min}}{L C_{\mathrm{T}}},
\label{eq:app_stepsize}
\end{equation}
then the iterates generated by Thermo-FL satisfy
\begin{align}
\frac{1}{T}
\sum_{t=1}^{T}
\mathbb{E}
\left[
\|\nabla J(W^{(t)})\|_2^2
\right]
\leq\;&
\frac{
2\left(\mathbb{E}[J(W^{(1)})]-J^\star\right)
}{
\eta\alpha_{\mathrm{T}}\rho_{\min}T
}
\nonumber\\
&+
\frac{
L\eta C_{\mathrm{T}}\sigma^2
}{
\alpha_{\mathrm{T}}\rho_{\min}
}
+
\frac{
2(\bar{\epsilon}+\varepsilon_{\mathrm{T}})
}{
\alpha_{\mathrm{T}}\rho_{\min}
}
+
\frac{
L\eta\nu_{\mathrm{T}}
}{
\alpha_{\mathrm{T}}\rho_{\min}
}.
\label{eq:app_main_bound}
\end{align}
\end{theorem}

\begin{proof}
By \(L\)-smoothness and the update \(W^{(t+1)}=W^{(t)}-\eta G^{(t)}\),
\begin{equation}
J(W^{(t+1)})
\leq
J(W^{(t)})
-
\eta
\left\langle
\nabla J(W^{(t)}),G^{(t)}
\right\rangle
+
\frac{L\eta^2}{2}
\|G^{(t)}\|_2^2 .
\label{eq:app_smooth_step}
\end{equation}
Taking expectations and applying the alignment and second-moment assumptions gives
\begin{align}
\mathbb{E}[J(W^{(t+1)})]
\leq\;&
\mathbb{E}[J(W^{(t)})]
-
\eta\alpha_{\mathrm{T}}\rho_{\min}
\mathbb{E}
\left[
\|\nabla J(W^{(t)})\|_2^2
\right]
+
\eta(\bar{\epsilon}+\varepsilon_{\mathrm{T}})
\nonumber\\
&+
\frac{L\eta^2}{2}
\left[
C_{\mathrm{T}}
\left(
\mathbb{E}
\left[
\|\nabla J(W^{(t)})\|_2^2
\right]
+\sigma^2
\right)
+
\nu_{\mathrm{T}}
\right].
\end{align}
Rearranging and using the stepsize condition in Eq.~\eqref{eq:app_stepsize},
\begin{align}
\mathbb{E}[J(W^{(t+1)})]
\leq\;&
\mathbb{E}[J(W^{(t)})]
-
\frac{\eta\alpha_{\mathrm{T}}\rho_{\min}}{2}
\mathbb{E}
\left[
\|\nabla J(W^{(t)})\|_2^2
\right]
\nonumber\\
&+
\eta(\bar{\epsilon}+\varepsilon_{\mathrm{T}})
+
\frac{L\eta^2}{2}
\left(
C_{\mathrm{T}}\sigma^2+\nu_{\mathrm{T}}
\right).
\end{align}
Summing over \(t=1,\ldots,T\), telescoping, using \(\mathbb{E}[J(W^{(T+1)})]\geq J^\star\), and dividing by \(T\eta\alpha_{\mathrm{T}}\rho_{\min}/2\) yields Eq.~\eqref{eq:app_main_bound}.
\end{proof}

\subsection{Interpretation}

The bound in Eq.~\eqref{eq:app_main_bound} has four terms. The first term decays as \(\mathcal{O}(1/T)\). The second term is the stochastic-gradient variance floor. The third term captures the bias introduced by thermal projection and by the abstract TERRA residual. The fourth term captures the second-moment contribution of the TERRA operator. Thus, Thermo-FL converges to a stationary neighborhood whose radius depends on stochastic noise, thermal adaptation, sparsification, and robust aggregation.

The theorem should not be read as an unconditional proof of convergence for arbitrary attacks. Rather, it states that if the implemented TERRA pipeline preserves sufficient alignment with the global descent direction and keeps the aggregate second moment bounded, then Thermo-FL retains the standard nonconvex stationary-point behavior expected of stochastic federated optimization. The bounded-influence inequalities above explain why the TERRA stages support these abstract conditions: high-magnitude updates are filtered, directionally inconsistent updates are rejected, retained coordinates are clipped, and sparse supports are aggregated without treating missing coordinates as zeros.
\end{document}